\documentclass[11pt,a4paper]{article}
\usepackage{times,latexsym}
\usepackage{url}
\usepackage{amsfonts}
\usepackage{amsmath}
\usepackage[T1]{fontenc}
\usepackage{times}
\usepackage{latexsym}
\usepackage{amsthm}

\theoremstyle{plain} 
\newtheorem{lemma}{Lemma}[section]

\theoremstyle{definition} 

\usepackage[table]{xcolor}
\usepackage{booktabs}
\usepackage{multirow}
\definecolor{lightblue}{rgb}{0.85, 0.9, 1} %
\theoremstyle{remark} 

\usepackage[table]{xcolor} 
\usepackage{colortbl}
\usepackage{microtype}
\usepackage{makecell}
\usepackage{hyperref}
\usepackage{graphicx}
\usepackage{multirow}
\usepackage{amsmath}
\usepackage{xcolor}
\usepackage{makecell}
\usepackage{booktabs}
\usepackage{caption}
\usepackage{subcaption}
\usepackage{amsfonts}
\usepackage{tabularx}
\usepackage{rotating}
\usepackage[T1]{fontenc}
\usepackage[utf8]{inputenc}
\usepackage{microtype}
\usepackage{inconsolata}
\usepackage{graphicx}
\usepackage{amssymb}
\usepackage{algorithmicx}
\usepackage{algorithm}
\usepackage{algpseudocode}
\usepackage{amsmath}
\usepackage{stfloats}     %

\usepackage{subcaption}
\usepackage{subcaption} 
\algnewcommand\algorithmicendwhile{\textbf{end\ while}}
\algrenewtext{While}[1]{\algorithmicwhile\ #1\ \algorithmicdo}
\usepackage{caption}
\usepackage{xcolor}
\newcommand{\legendoutline}[2][1.5ex]{%
  \begingroup
  \setlength{\fboxsep}{0pt}%
  \setlength{\fboxrule}{0.8pt}%
  \fcolorbox{#2}{white}{\phantom{\rule{#1}{#1}}}%
  \endgroup
}
\usepackage[table]{xcolor}

\definecolor{lightblue}{RGB}{203, 204, 255}
\definecolor{lightgray}{RGB}{229, 229, 229}
\usepackage[acceptedWithA]{tacl2021v1}
\usepackage[]{tacl2021v1}

\usepackage{xspace,mfirstuc,tabulary}

\newif\iftaclinstructions
\taclinstructionsfalse 
\iftaclinstructions
\renewcommand{\confidential}{}
\renewcommand{\anonsubtext}{(No author info supplied here, for consistency with
TACL-submission anonymization requirements)}
\newcommand{\instr}
\fi

\iftaclpubformat 

\else

\fi

\title{Event Interaction in Low-Rank Bottlenecks for Temporal Relation Extraction}

\author{Wei Sun, Tingyu Qu, Jesse Davis \and Marie-Francine Moens \\
        Department of Computer Science, KU Leuven \\
        Celestijnenlaan 200A 3001 Heverlee, Belgium \\
        \texttt{\{sun.wei, tingyu.qu, jesse.davis, sien.moens\}@kuleuven.be}}

\date{}

\begin{document}
\maketitle
\begin{abstract}

Temporal relation extraction determines whether an event occurs before, after, or simultaneously with another event. Therefore, it relies on accurately modeling how the two events interact. 
Mainstream systems achieve this by concatenating event spans or using shallow fusion, which works well when all model parameters are trainable. 
However, in parameter-efficient fine-tuning, low-rank bottlenecks restrict information flow and prevent these interaction signals from passing through, leading to clear performance drops. 
To address this limitation, we propose a theoretically grounded architecture, Convolutional Bottleneck Interaction (CBI), which first applies lightweight depthwise convolution to enhance event representations and then uses element-wise multiplication to capture effective event-event interactions inside the bottleneck. Across five datasets and seven backbone models in the Adapter and LoRA settings, CBI provides consistent and substantial gains, up to +31.7 micro F1, while adding minimal computational cost, showing that explicit interaction inside low-rank spaces is crucial for temporal relation extraction. The code is available at \url{https://github.com/VRCMF/CIF.git}.
\end{abstract}

\section{Introduction}
\label{sec:intro}

\begin{figure*}[ht]
\centering
\subfloat[\tiny{Full fine-tuning}]{\includegraphics[width=0.65\textwidth]{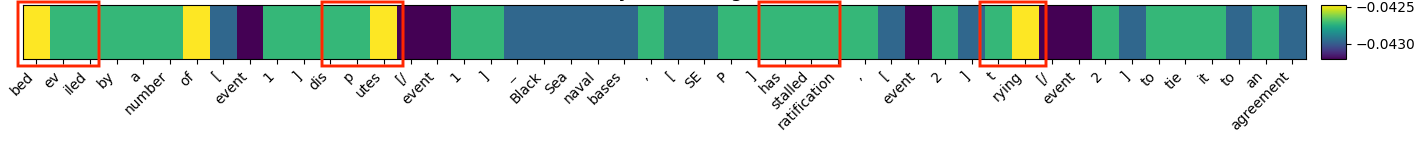}
\label{fig:full_ft_sample_111}
}
\hfil
\subfloat[\tiny{Text concatenation with Adapter fine-tuning.}]{\includegraphics[width=0.65\textwidth]{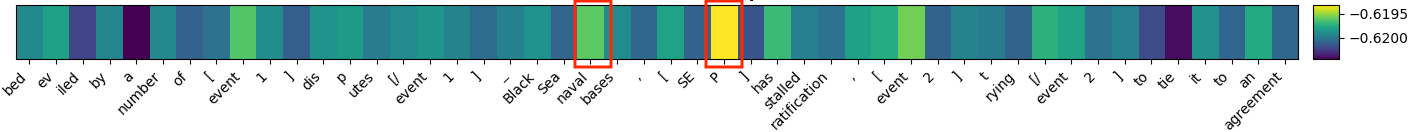}
\label{fig:text_concatenation_adapter}
}
\hfil
\subfloat[\tiny{Text concatenation with LoRA fine-tuning.}]{\includegraphics[width=0.65\textwidth]{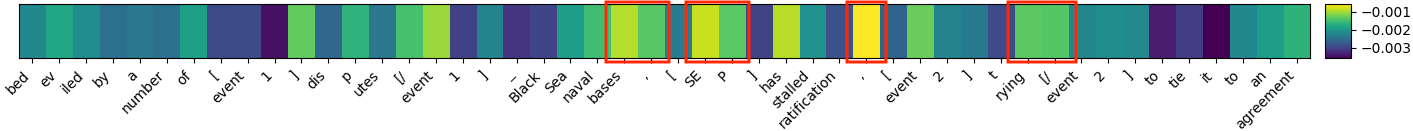}
\label{fig:text_concatenation_lora}
}
\hfil
\subfloat[\tiny{Embedding Fusion with Adapter fine-tuning.}]{\includegraphics[width=0.4\textwidth]{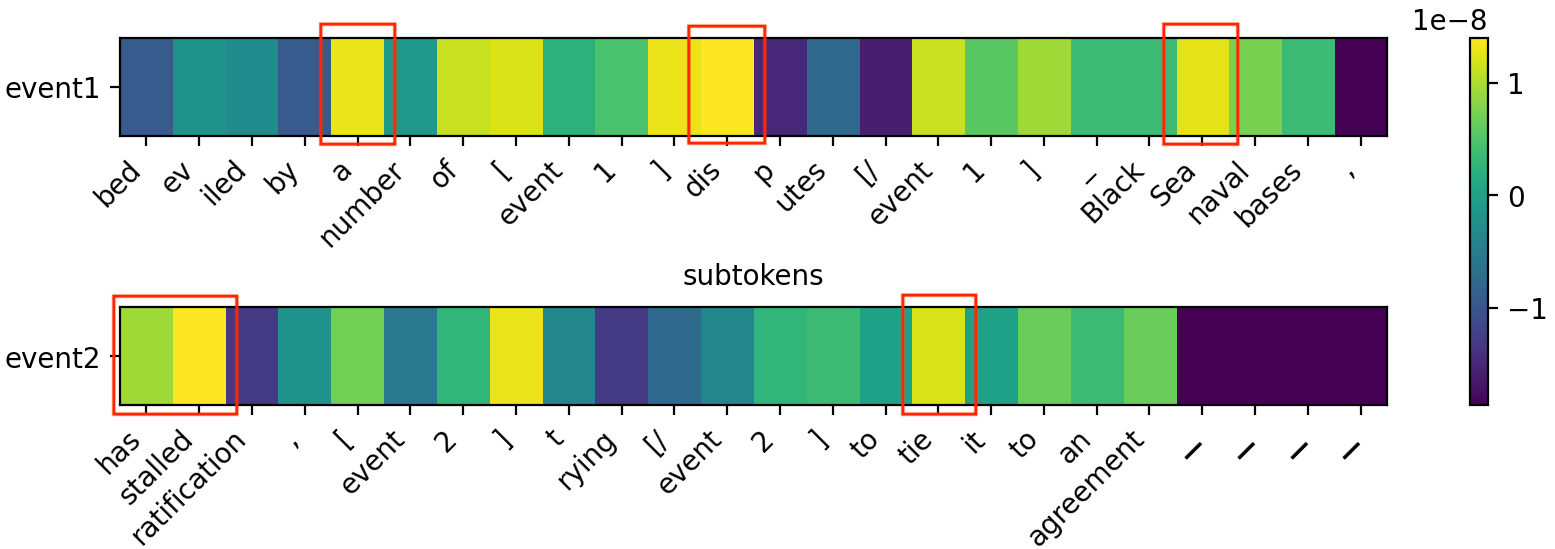}
\label{fig:embedding_adapter}
}
\hfil
\subfloat[\tiny{Embedding Fusion with LoRA fine-tuning.}]{\includegraphics[width=0.4\textwidth]{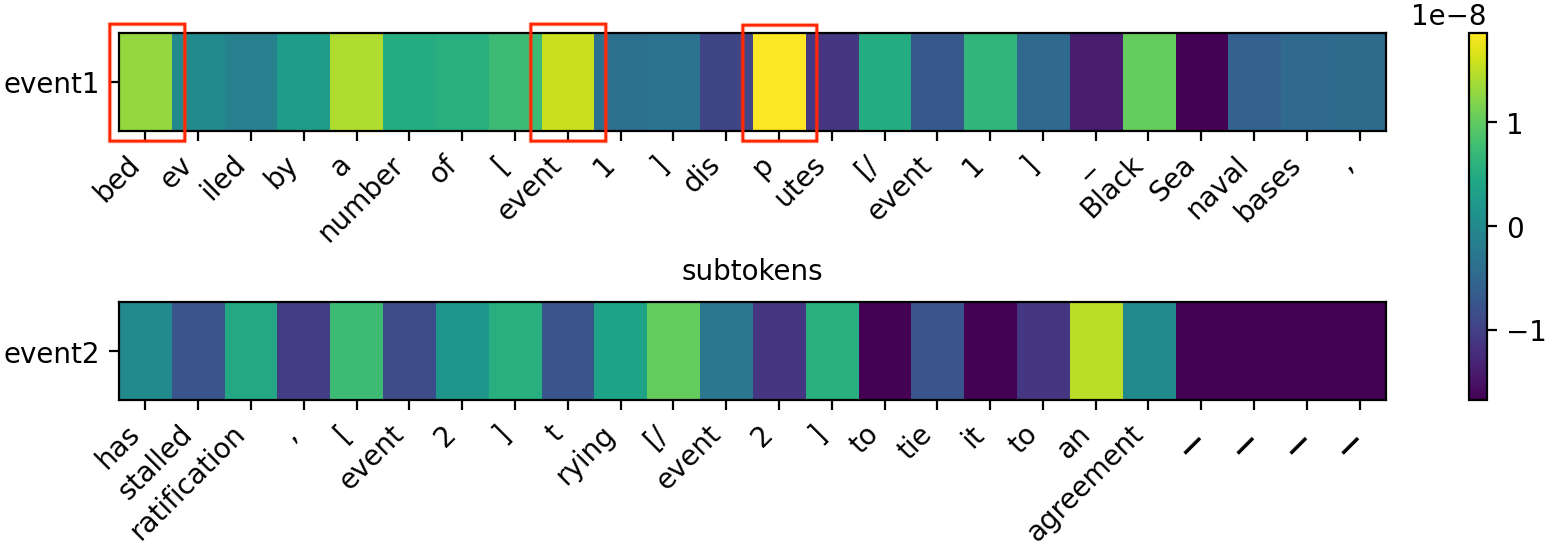}
\label{fig:embedding_lora}
}
\hfil
\caption{Visualization of event interaction maps of the temporal relation extraction task for (a) \textbf{Full Fine-tuning}, \textbf{Text Concatenation} with (b) \textbf{Adapter} fine-tuning, (c) \textbf{LoRA} fine-tuning,  \textbf{Embedding Fusion} with (d) \textbf{Adapter} fine-tuning, and (e) \textbf{LoRA} fine-tuning .
\legendoutline{red} marks tokens emphasized by different methods.
\textcolor{black}{Attention in Text Concatenation and Embedding Fusion is diffused over irrelevant tokens, limiting temporal cue propagation through the bottleneck, while full fine-tuning captures these cues effectively but at a high computational cost.}
}
\label{fig:adapter_original}
\end{figure*}

Temporal relations indicate whether an event occurs before, after, or simultaneously with another event. These relations are important for narrative understanding~\citep{yao-huang-2018-temporal,ballesteros-etal-2020-severing}, story generation~\citep{ye23neural,wang-etal-2025-generating}, and answering temporal questions~\citep{qin-etal-2021-timedial,roccabruna2024will}. \textcolor{black}{Inaccurate temporal predictions can propagate through these systems and lead to incoherent or misleading task outputs.}  The goal of temporal relation extraction is to automatically identify temporal event relations from natural language text~\citep{ning-etal-2018-improving,leeuwenberg2019survey}. Although several architectures have been explored, many temporal relation extraction systems use BERT- and RoBERTa-based encoders~\citep{zhou-etal-2022-rsgt,man2022selecting,huang-etal-2023-classification}.

{A current approach is to fine-tune large language models to perform temporal relational extraction. Because of the computational cost, it would be desirable to employ parameter-efficient fine-tuning~(PEFT) methods such as adapters~\citep{pfeiffer-etal-2021-adapterfusion} or LoRA~\citep{hu2022lora}. 
\textcolor{black}{While PEFT reduces training cost, it also constrains learning by limiting updates to a small set of low-rank parameters. 
As a result, task-specific information must be encoded within a restricted subspace. 
Temporal relation extraction, which depends on fine-grained interactions between events, is particularly sensitive to this constraint. 
In this work, we investigate this limitation and propose methods to mitigate it.}

Unfortunately, PEFT suffers from a key limitation in this context: all task-specific information must pass through low-rank bottlenecks. This requirement is at odds with how most existing approaches model event interaction: they concatenate the two event spans and feed the combined representation into a classifier~\citep{ballesteros-etal-2020-severing,fangcross,wen-ji-2021-utilizing,zhou-etal-2022-rsgt,man2022selecting,huang-etal-2023-classification,roccabruna2024will}, which we refer to as Text Concatenation. They also augment the concatenated representation by performing shallow fusion via an element-wise product at the input-level~\citep{chaturvedi-etal-2025-temporal}, called Embedding Fusion below. These operations work well when the entire model is trainable. 
In contrast, in a PEFT setting only a small set of low-rank parameters is trainable.
Although cross-event interactions are computed by the frozen backbone, the model can adjust them only through low-rank bottlenecks, which limits its ability to learn task-specific interaction patterns.
To show this limitation, consider an example from the TDD-Man dataset~\citep{naik2019tddiscourse}:
Sentence 1: \textit{... bedeviled by a number of [event1]disputes[/event1] - Black Sea naval bases.} Sentence 2: \textit{... has stalled ratification, [event2]trying[/event2] to tie it to an agreement}.
\textcolor{black}{The gold relation is \textit{INCLUDES} because ``bedeviled'' and ``has stalled'' both signal an ongoing background situation, ``disputes'' denotes a temporally extended state, and ``trying'' introduces a more local action unfolding within that state, so the time span of ``disputes'' contains the time of ``trying''.
Figure~\ref{fig:adapter_original} shows that in both Text Concatenation and Embedding Fusion methods, attention is dispersed over irrelevant tokens instead of concentrating on tokens representing the temporal cues. This indicates that crucial evidence fails to propagate effectively through the bottlenecked representation space. In contrast, full fine-tuning can successfully capture these important cues. 
However, it requires markedly more computational resources and longer training time, as shown in Table~\ref{tab:cost_time}.}

{Based on this observation, we propose Convolutional Bottleneck Interaction (CBI) as a lightweight and theoretically grounded architecture that can be integrated into existing PEFT schemes. CBI has two key components.  First, it directly models event–event interactions within the bottleneck by fusing the hidden states of each event. We test three operations: multiplication, addition, and subtraction. 
Multiplication performs the best because it captures channel-wise co-activation and approximates diagonal bilinear interactions. Because each event is encoded independently, contextual cues such as tense or discourse markers are often misaligned, and the element-wise multiplication pairs informative features with uninformative ones, amplifying noise that low-rank modules cannot correct. Consequently, our second component involves performing a lightweight depthwise convolution before the interaction step. 
This convolution enlarges the receptive field of each token, thus capturing nearby temporal cues, smoothing noise, and reducing the sensitivity to small positional shifts of tokens}.

We evaluate our method on five temporal relation extraction datasets: TimeBank-Dense~\citep{cassidy-etal-2014-annotation}, MATRES~\citep{ning-etal-2018-improving}, TCR~\citep{ning-etal-2018-joint}, TDD-Man~\citep{naik2019tddiscourse}, and TDD-Auto~\citep{naik2019tddiscourse}, using seven backbone models: BERT-large~\citep{devlin2019bert}, RoBERTa-large~\citep{liu2019roberta}, Gemma2-2B~\citep{team2024gemma}, Phi3.5-3.8B~\citep{abdin2024phi}, Mistral-7B~\citep{Jiang2023Mistral7}, Qwen2.5-7B~\citep{bai2023qwen}, and Llama 3.1-8B~\citep{touvron2023llama}. 
Under both adapter and LoRA configurations, CBI achieves consistent improvements, reaching gains of up to +31.7 micro-F1 while adding negligible computational and memory cost. 
These results show that explicitly modeling event interaction inside low-rank bottlenecks is both effective and necessary, and that convolution-enhanced multiplicative interaction provides a simple and effective solution for temporal relation extraction in a PEFT setting.
Attention in Text Concatenation and Embedding Fusion is diffused over irrelevant tokens, limiting temporal cue propagation through the bottleneck, while full fine-tuning captures these cues effectively but at a high computational cost.

We make the following contributions: 
\begin{itemize}
\item We show why PEFT underperforms in temporal relation extraction: low-rank bottlenecks miss key cross-event interactions, and common text concatenation or shallow fusion fails to pass temporal cues through a bottlenecked representation space.

\item We propose an event interaction PEFT framework and find multiplication to be the most effective interaction. 
We analyze its weaknesses: token misalignment and noise amplification, and then introduce a lightweight depthwise convolution to stabilize and enhance multiplicative interaction by enlarging the receptive field and smoothing local noise.

\item We propose a theoretically grounded architecture called Convolutional Bottleneck Interaction (CBI), which applies convolution followed by element-wise multiplication inside the bottleneck space. Across five datasets and seven backbone models in adapter and LoRA settings, CBI produces large and consistent improvements, up to +31.7 micro-F1, with minimal computation.
\end{itemize}

\section{Related Work}

\subsection{Temporal Relation Extraction}

Pretrained language models improved temporal relation extraction by providing rich contextual event representations.
The MATRES dataset~\citep{ning-etal-2018-multi} makes this requirement explicit because its densely annotated event pairs can only be predicted correctly when models capture contextual and co-occurrence cues.
Contextual encoders such as ELMo~\citep{peters-etal-2018-deep}, BERT~\citep{devlin2019bert}, and RoBERTa~\citep{liu2019roberta} naturally support joint modeling of two segments.
Thus, the task was widely cast as sentence-pair classification: mark two events, feed their combined context into the encoder, and classify their temporal relation.

This formulation produced a mainstream architecture in which concatenated event contexts are encoded jointly so that self-attention can model event interactions.
Many extensions follow this paradigm~\citep{ballesteros-etal-2020-severing,fangcross,wen-ji-2021-utilizing,zhou-etal-2022-rsgt,man2022selecting,huang-etal-2023-classification}, while including multitask learning, joint prediction of order and relation types, and graph-based modeling for document or cross-document consistency.
\citet{chaturvedi-etal-2025-temporal} further strengthen pairwise fusion signals by introducing an element-wise product of paired event embeddings at the input level.

All these methods assume full-model fine-tuning, where self-attention layers freely adjust parameters to encode event–event interaction.
Under parameter-efficient fine-tuning, most parameters are frozen, and information must pass through low-rank bottlenecks.
These bottlenecks limit the ability of the model to adjust representations for encoding event interaction, making concatenation or other shallow fusion strategies insufficient for transmitting detailed temporal cues.

\subsection{Parameter-Efficient Fine-Tuning}

Large language models (LLMs) have become widely used across NLP because of their strong generalization ability. In temporal relation extraction, 
LLMs have mainly been explored through zero- and few-shot prompting rather than fine-tuning. \citet{chan-etal-2024-exploring} showed that ChatGPT struggles with temporal ordering, and \citet{yuan-etal-2023-zero} proposed prompt decomposition for zero-shot extraction but still 
observed substantial gaps compared to supervised models. 
Recent work has explored parameter-efficient fine-tuning (PEFT) methods such as LoRA~\citep{hu2022lora}. \citet{roccabruna2024will} demonstrated that LoRA yields noticeable gains over instruction-tuned LLMs, suggesting that parameter-efficient updates offer a promising direction.
\textcolor{black}{While prior work such as~\citet{roccabruna2024will} demonstrates the effectiveness of LoRA for temporal relation extraction, it focuses on a single PEFT method in limited experimental settings.}

PEFT methods such as Adapters~\citep{pfeiffer-etal-2021-adapterfusion,lu2023uniadapter} and LoRA have been shown to be effective in many NLP tasks by significantly reducing the number of trainable parameters~\citep{ben-zaken-etal-2022-bitfit,chen2023longlora,wang-etal-2022-adamix}. Yet, their applicability to temporal relation extraction remains unclear, especially because common concatenation-based architectures were originally designed for fully trainable models rather than for settings where most parameters are frozen. To address this gap, we 
fine-tune LLMs using Adapters and LoRA across five temporal relation extraction datasets, providing the first comprehensive evaluation of PEFT methods for this task.

\subsection{Modeling Cross-Event Interactions}

Temporal relation extraction systems differ mainly in how they model interactions between two events. Most adopt implicit interaction: they concatenate the two event spans and rely on a transformer to infer cross-event dependencies through self-attention. This requires a sufficient model capacity. However, in a PEFT setting low-rank adapters constrain the representational space and limit how much interaction-related information can pass through, making implicit modeling unreliable.

To overcome this bottleneck limitation, we draw inspiration from vision–language PEFT models. 
\citet{qu2024introducing} introduced feature routing functions that enhance cross-modal alignment by directing information through low-rank bottlenecks, a strategy later adopted by \citet{guo2025wander} and \citet{borchert-etal-2025-language}. The latter applied this approach to low-resource cross-lingual transfer, consistently outperforming naïve concatenation.

Following this idea, we design three explicit event interaction functions within PEFT bottlenecks: element-wise multiplication, addition, and subtraction of event representations. 
Multiplication performs best because it captures channel-wise co-activation patterns essential for temporal relation extraction. However, it is sensitive to representation misalignment and noise.
To mitigate this, we introduce a depthwise convolution operation, which performs token-level alignment and interaction within the PEFT bottlenecks. 
To our knowledge, our paper proposes the first temporal relation extraction method to implement explicit event interaction in PEFT settings instead of relying on implicit event interactions.

\section{Method}
\label{sec:method}

\textcolor{black}{While LoRA and Adapters efficiently introduce low-rank updates, they primarily operate on each input independently and lack explicit mechanisms to model structured interactions between multiple inputs (e.g., event pairs). To address this limitation, we propose Convolutional Bottleneck Interaction, which augments the PEFT bottleneck with explicit cross-event interaction modeling while preserving parameter efficiency.}

\subsection{Parameter-Efficient Adaptation via LoRA and Adapters}

Before introducing our CBI adapter, we briefly review two common PEFT strategies shown in Figure~\ref{fig:lora_arch} and Figure~\ref{fig:adapter_arch}.

\paragraph{LoRA:}
LoRA adds trainable low-rank matrices to key and value projections while keeping pretrained weights frozen.
As shown in Figure~\ref{fig:lora_arch}, each projection (e.g., $\mathbf{W}_Q$, $\mathbf{W}_V$) receives a rank-$r$ update:
\begin{equation}
    \Delta \mathbf{W} = \mathbf{W}_{\text{down}} \mathbf{W}_{\text{up}},
\end{equation}
where $\mathbf{W}_{\text{down}}\in\mathbb{R}^{d\times r}$ is the down-projection and
$\mathbf{W}_{\text{up}}\in\mathbb{R}^{r\times d}$ is the up-projection, with $r\ll d$. 
$d$ is the dimension of hidden representations. 
Only these two matrices are trained.

\paragraph{Adapters:}
Adapters insert a small bottleneck module inside each transformer block (Figure~\ref{fig:adapter_arch}).
Given hidden state $\mathbf{H}$:
\begin{equation}
\operatorname{Adapter}(\mathbf{H}) =
\mathbf{H} + \mathbf{W}_{\text{up}} 
\mathbin \sigma(\mathbf{W}_{\text{down}} \mathbf{H}),
\end{equation}
where $\sigma(\cdot)$ is the activation function, and  
$\mathbf{W}_{\text{down}}$, $\mathbf{W}_{\text{up}}$ form the learnable low-rank bottleneck.

\subsection{Overview of CBI}

\begin{figure}[ht]
\centering
\begin{subfigure}[b]{0.5\textwidth}
  \centering
  \includegraphics[width=\linewidth]{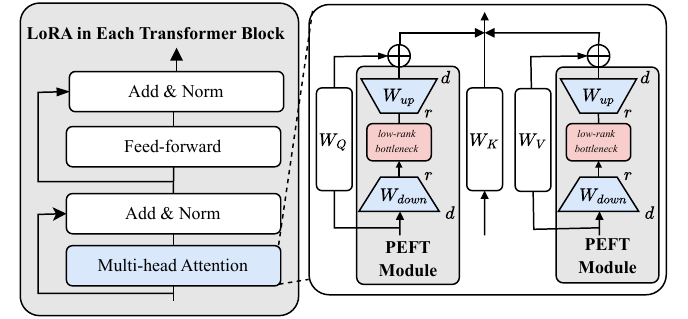}
  \caption{\tiny LoRA in a transformer block.}
  \label{fig:lora_arch}
\end{subfigure}
\hfill
\begin{subfigure}[b]{0.32\textwidth}
  \centering
  \includegraphics[width=0.95\linewidth]{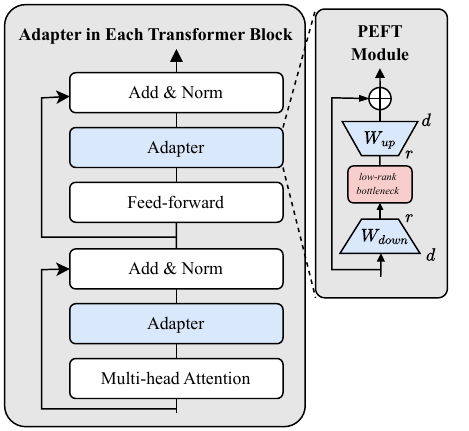}
  \caption{\tiny Adapter in a transformer block.}
  \label{fig:adapter_arch}
\end{subfigure}
\hfill
\begin{subfigure}[b]{0.15\textwidth} 
  \centering
  \includegraphics[width=\linewidth]{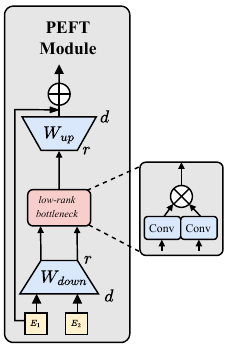} 
  \caption{\tiny Convolutional Bottleneck Interaction.}
  \label{fig:dci_arch}
\end{subfigure}
\caption{Illustration of (a) LoRA, 
(b) Adapter, and 
(c) The Convolutional Bottleneck Interaction adapter, which respectively replaces the LoRA or adapter block when integrated in the transformer block of (a) and (b). 
The parameters in the blue boxes are learned during the fine-tuning.}
\label{fig:overview_arch}
\end{figure}

The \textbf{Convolutional Bottleneck Interaction (CBI) Adapter} (Figure~\ref{fig:dci_arch})
is inserted inside each low-rank bottleneck.
Given two events, CBI enhances bottleneck features and explicitly models cross-event exchange through:
(a) Dual-Tower Local Enrichment: down-project each event into a rank-$r$ bottleneck and apply depthwise convolution to aggregate local temporal context;
(b) Explicit Cross-Event Interaction: interact the enriched bottleneck features via element-wise multiplication to capture cross-event interaction.

A complete theoretical analysis of CBI is provided in Appendix~\ref{appendix:theory}.

\subsection{Notation and Input Representations}

Let $\mathbf{H}_1,\mathbf{H}_2\in\mathbb{R}^{L\times d}$ be contextualized embeddings
of two textual segments containing events $e_1$ and $e_2$, produced by a frozen backbone $f_\theta$.
Both sequences are padded to length $L$.
CBI operates only on these two events within the PEFT bottleneck.

\subsection{Dual-Tower Local Enrichment}

We first map each event to the bottleneck space:
\begin{multline}
\mathbf{E}_1=\mathbf{H}_1\mathbf{W}_{\text{down}},\quad 
\mathbf{E}_2=\mathbf{H}_2\mathbf{W}_{\text{down}}
\end{multline}
where $\mathbf{E}_i\in\mathbb{R}^{L\times r}$ contains $r$ latent channels~(\textcolor{black}{a channel is a dimension in the bottleneck representation}) per token. $\mathbf{W}_{\text{down}}\in\mathbb{R}^{d\times r},\ r\ll d$.

In a rank-$r$ bottleneck, each channel is narrow and may miss local context.
We then apply a depthwise convolution independently to each channel:
\begin{equation}
\tilde{\mathbf{E}}_i=\mathrm{DWConv}_i(\mathbf{E}_i),\quad i\in\{1,2\},
\end{equation}
where $\mathrm{DWConv}_i$ 
uses a 1D convolution kernel $\boldsymbol{\theta}=[\theta_{-k/2},\ldots,\theta_{k/2}]\in\mathbb{R}^{k}$ 
along the token axis.
This operation can be equivalently expressed as a matrix multiplication,
where $\mathbf{T}\in\mathbb{R}^{L\times L}$ is the Toeplitz matrix induced by the kernel $\boldsymbol{\theta}$.
For each event:
\begin{multline}\label{eq:conv_def}
    \tilde{\mathbf{E}}_i = \mathbf{T}\mathbf{E}_i \in\mathbb{R}^{L\times r},
    \quad\text{equivalently}\quad \\
    \tilde{\mathbf{e}}_i(l)=\sum_{j} \theta_j\,\mathbf{e}_i(l+j).
\end{multline}
Depthwise convolution expands the local receptive field and smooths noise,
yielding enriched bottleneck features $\tilde{\mathbf{E}}_i\in\mathbb{R}^{L\times r}$.

\subsection{Explicit Cross-Event Interaction}

After enrichment, we form cross-event interaction features by element-wise multiplication:
\begin{equation}
\mathbf{F} = \tilde{\mathbf{E}}_1 \odot \tilde{\mathbf{E}}_2,
\qquad
\mathbf{F}\in\mathbb{R}^{L\times r}.
\label{interaction_method}
\end{equation}
For each token position $l$ and channel $c$,
\begin{equation}
\mathbf{F}[l,c]=\tilde{\mathbf{E}}_1[l,c]\tilde{\mathbf{E}}_2[l,c].
\end{equation}
The product $\mathbf{F}[l,c]$ is a second-order event interaction.
Element-wise products over tokens in a fixed channel recover the diagonal bilinear term (\textcolor{black}{i.e., the result of a bilinear interaction that only involves element-wise interactions between corresponding dimensions}):
\begin{equation}
\sum_{l=1}^L \mathbf{F}[l,c]
=
\left(\tilde{\mathbf{E}}_1^\top\tilde{\mathbf{E}}_2\right)_{c,c}.
\end{equation}
Therefore, $\mathbf{F}$ is an explicit token-wise representation of diagonal bilinear
cross-event interactions, which are hard to learn implicitly under a rank-$r$ PEFT update.
Depthwise convolution with multiplication further allows near-aligned tokens
to interact even if their exact positions differ.

\textcolor{black}{This benefit of the convolution is particularly important because element-wise multiplication implicitly relies on token-level alignment: even noise coming from small positional deviations can perturb the product and lead to unstable second-order interactions (theoretical proof is in Appendix~\ref{sec:multi_error}). By introducing a depthwise convolution prior to interaction, CBI mitigates this misalignment through local aggregation, which acts as a soft alignment and denoising operator. 
This preserves the desired diagonal bilinear signal while suppressing noise, resulting in more robust cross-event interaction features; see Appendix~\ref{sec:cbi_solve_error} for the theoretical proof.
}

\subsection{Residual Injection}

We convert the interaction field back to hidden size $d$
using the PEFT up-projection and inject it through a residual path:
\begin{equation}
\Delta\mathbf{H} = \mathbf{F}\mathbf{W}_{\text{up}},
\qquad
\mathbf{H}'_1 = \mathbf{H}_1 + \Delta\mathbf{H},
\end{equation}
where $\mathbf{W}_{\text{up}}\in\mathbb{R}^{r\times d}$ is the PEFT up-projection.
$\mathbf{H}'_1$ is then passed to the next transformer blocks.

\begin{table*}[]
\centering
\tiny
\renewcommand{\arraystretch}{1}
\setlength{\tabcolsep}{8pt}
\begin{tabular}{l l| l| c c c c c| c}
\toprule
\textbf{Backbone} & \textbf{PEFT} & \textbf{Methods} & \textbf{MATRES} & \textbf{TB-Dense} & \textbf{TCR} & \textbf{TDD-Auto} & \textbf{TDD-Man} & \textbf{AVG} \\
\midrule
& \multirow{1}{*}{/} 
& \textcolor{black}{FFT~(upper bound)} & 82.15\textsubscript{$\pm$0.34} & 63.27\textsubscript{$\pm$0.41} & 81.12\textsubscript{$\pm$0.38} & 86.71\textsubscript{$\pm$0.33} & 80.82\textsubscript{$\pm$0.45} & 78.81\textsubscript{$\pm$0.36} \\ 
\cmidrule(lr){2-9}
\multirow{6}{*}{\textbf{RoBERTa-large}} 
& \multirow{3}{*}{Adapter} 
& Text Concatenation & \cellcolor{lightgray}{\underline{58.31}}\textsubscript{$\pm$0.91} & \cellcolor{lightgray}{\underline{41.69}}\textsubscript{$\pm$1.07} & \cellcolor{lightgray}{\underline{59.72}}\textsubscript{$\pm$0.66} & \cellcolor{lightgray}{\underline{32.85}}\textsubscript{$\pm$1.12} & {{31.59}}\textsubscript{$\pm$0.53} & {{44.83}}\textsubscript{$\pm$0.78} \\
& & Embedding Fusion & 58.31\textsubscript{$\pm$0.74} & 41.69\textsubscript{$\pm$0.88} & 59.72\textsubscript{$\pm$1.13} & 32.85\textsubscript{$\pm$0.97} & \cellcolor{lightgray}{\underline{57.14}}\textsubscript{$\pm$0.49} & \cellcolor{lightgray}{\underline{49.94}}\textsubscript{$\pm$1.05} \\
& & \textbf{CBI} & \cellcolor{lightblue}{\textbf{74.35}}\textsubscript{$\pm$0.85} & \cellcolor{lightblue}{\textbf{68.18}}\textsubscript{$\pm$1.07} & \cellcolor{lightblue}{\textbf{73.92}}\textsubscript{$\pm$1.61} & \cellcolor{lightblue}{\textbf{86.85}}\textsubscript{$\pm$1.32} & \cellcolor{lightblue}{\textbf{79.67}}\textsubscript{$\pm$1.18} & \cellcolor{lightblue}{\textbf{76.59}}\textsubscript{$\pm$1.22} \\
\cmidrule(lr){2-9}
& \multirow{3}{*}{LoRA} 
& Text Concatenation & \cellcolor{lightgray}{\underline{67.58}}\textsubscript{$\pm$0.62} & \cellcolor{lightgray}{\underline{55.17}}\textsubscript{$\pm$0.97} & \cellcolor{lightgray}{\underline{67.44}}\textsubscript{$\pm$0.74} & \cellcolor{lightgray}{\underline{71.85}}\textsubscript{$\pm$0.91} & 64.57\textsubscript{$\pm$0.83} & \cellcolor{lightgray}{\underline{65.32}}\textsubscript{$\pm$0.58} \\
& & Embedding Fusion & 59.52\textsubscript{$\pm$1.01} & 46.08\textsubscript{$\pm$0.79} & 62.35\textsubscript{$\pm$1.18} & 64.95\textsubscript{$\pm$0.55} & \cellcolor{lightgray}{\underline{68.24}}\textsubscript{$\pm$0.47} & 60.23\textsubscript{$\pm$0.88} \\
& & \textbf{CBI} & \cellcolor{lightblue}{\textbf{70.16}}\textsubscript{$\pm$1.38} & \cellcolor{lightblue}{\textbf{62.70}}\textsubscript{$\pm$0.74} & \cellcolor{lightblue}{\textbf{67.75}}\textsubscript{$\pm$1.20} & \cellcolor{lightblue}{\textbf{89.35}}\textsubscript{$\pm$0.83} & \cellcolor{lightblue}{\textbf{76.24}}\textsubscript{$\pm$1.27} & \cellcolor{lightblue}{\textbf{73.24}}\textsubscript{$\pm$1.00} \\
\midrule
& \multirow{1}{*}{/} 
& \textcolor{black}{FFT~(upper bound)} & 81.76\textsubscript{$\pm$0.32} & 60.78\textsubscript{$\pm$0.41} & 79.35\textsubscript{$\pm$0.36} & 81.77\textsubscript{$\pm$0.33} & 80.53\textsubscript{$\pm$0.39} & 76.84\textsubscript{$\pm$0.35} \\
\cmidrule(lr){2-9}
\multirow{6}{*}{\textbf{BERT-large}} 
& \multirow{3}{*}{Adapter} 
& Text Concatenation & \cellcolor{lightgray}{\underline{58.31}}\textsubscript{$\pm$0.72} & \cellcolor{lightgray}{\underline{41.69}}\textsubscript{$\pm$1.31} & \cellcolor{lightgray}{\underline{59.72}}\textsubscript{$\pm$0.91} & \cellcolor{lightgray}{\underline{32.85}}\textsubscript{$\pm$0.59} & 31.59\textsubscript{$\pm$0.94} & 44.83\textsubscript{$\pm$0.86} \\
& & Embedding Fusion & 58.31\textsubscript{$\pm$0.88} & 41.69\textsubscript{$\pm$0.97} & 59.72\textsubscript{$\pm$1.13} & 32.85\textsubscript{$\pm$0.52} & \cellcolor{lightgray}{\underline{58.37}}\textsubscript{$\pm$0.43} & \cellcolor{lightgray}{\underline{50.19}}\textsubscript{$\pm$1.02} \\
& & \textbf{CBI} & \cellcolor{lightblue}{\textbf{71.69}}\textsubscript{$\pm$0.97} & \cellcolor{lightblue}{\textbf{67.55}}\textsubscript{$\pm$1.08} & \cellcolor{lightblue}{\textbf{72.53}}\textsubscript{$\pm$1.51} & \cellcolor{lightblue}{\textbf{90.60}}\textsubscript{$\pm$1.44} & \cellcolor{lightblue}{\textbf{79.02}}\textsubscript{$\pm$1.19} & \cellcolor{lightblue}{\textbf{76.28}}\textsubscript{$\pm$1.21} \\
\cmidrule(lr){2-9}
& \multirow{3}{*}{LoRA} 
& Text Concatenation & \cellcolor{lightgray}{\underline{66.69}}\textsubscript{$\pm$0.59} & \cellcolor{lightgray}{\underline{54.86}}\textsubscript{$\pm$0.98} & \cellcolor{lightgray}{\underline{65.90}}\textsubscript{$\pm$1.04} & \cellcolor{lightgray}{\underline{71.65}}\textsubscript{$\pm$1.22} & \cellcolor{lightgray}{\underline{65.55}}\textsubscript{$\pm$0.47} & \cellcolor{lightgray}{\underline{64.93}}\textsubscript{$\pm$0.66} \\
& & Embedding Fusion & 58.31\textsubscript{$\pm$1.12} & 45.30\textsubscript{$\pm$0.76} & 61.11\textsubscript{$\pm$0.83} & 63.75\textsubscript{$\pm$1.09} & 60.98\textsubscript{$\pm$0.61} & 57.89\textsubscript{$\pm$0.95} \\
& & \textbf{CBI} & \cellcolor{lightblue}{\textbf{71.94}}\textsubscript{$\pm$1.05} & \cellcolor{lightblue}{\textbf{65.36}}\textsubscript{$\pm$0.89} & \cellcolor{lightblue}{\textbf{68.21}}\textsubscript{$\pm$0.62} & \cellcolor{lightblue}{\textbf{88.25}}\textsubscript{$\pm$1.19} & \cellcolor{lightblue}{\textbf{76.82}}\textsubscript{$\pm$0.53} & \cellcolor{lightblue}{\textbf{74.12}}\textsubscript{$\pm$0.76} \\
\midrule
& \multirow{1}{*}{/} 
& \textcolor{black}{FFT~(upper bound)} 
& 78.35\textsubscript{$\pm$0.39} 
& 73.36\textsubscript{$\pm$0.61} 
& 73.63\textsubscript{$\pm$0.67} 
& 89.74\textsubscript{$\pm$0.45} 
& 76.23\textsubscript{$\pm$0.57} 
& 78.26\textsubscript{$\pm$0.53} \\
& \multirow{1}{*}{/}  
& \textcolor{black}{ZST~(lower bound)} 
& 53.56\textsubscript{$\pm$1.88} 
& 25.78\textsubscript{$\pm$2.32} 
& 53.62\textsubscript{$\pm$2.01} 
& 24.79\textsubscript{$\pm$2.41} 
& 26.79\textsubscript{$\pm$2.14} 
& 36.91\textsubscript{$\pm$1.93} \\
\cmidrule(lr){2-9}
\multirow{2}{*}{\textbf{Gemma2-2B}} 
& \multirow{3}{*}{Adapter} 
& Text Concatenation & \cellcolor{lightgray}{\underline{59.19}}\textsubscript{$\pm$0.63} & \cellcolor{lightgray}{\underline{44.51}}\textsubscript{$\pm$1.32} & \cellcolor{lightgray}{\underline{59.10}}\textsubscript{$\pm$1.04} & \cellcolor{lightgray}{\underline{52.05}}\textsubscript{$\pm$0.77} & \cellcolor{lightgray}{\underline{59.27}}\textsubscript{$\pm$0.81} & \cellcolor{lightgray}{\underline{54.82}}\textsubscript{$\pm$0.92} \\
& & Embedding Fusion & 53.55\textsubscript{$\pm$1.18} & 39.18\textsubscript{$\pm$0.74} & 55.71\textsubscript{$\pm$0.86} & 31.65\textsubscript{$\pm$0.49} & 25.80\textsubscript{$\pm$1.29} & 41.18\textsubscript{$\pm$1.03} \\
& & \textbf{CBI} & \cellcolor{lightblue}{\textbf{62.58}}\textsubscript{$\pm$1.02} & \cellcolor{lightblue}{\textbf{59.09}}\textsubscript{$\pm$1.14} & \cellcolor{lightblue}{\textbf{64.20}}\textsubscript{$\pm$1.48} & \cellcolor{lightblue}{\textbf{77.20}}\textsubscript{$\pm$1.55} & \cellcolor{lightblue}{\textbf{68.24}}\textsubscript{$\pm$1.27} & \cellcolor{lightblue}{\textbf{66.26}}\textsubscript{$\pm$1.29} \\
\cmidrule(lr){2-9}
& \multirow{3}{*}{LoRA} 
& Text Concatenation & \cellcolor{lightgray}{\underline{60.40}}\textsubscript{$\pm$0.74} & \cellcolor{lightgray}{\underline{49.69}}\textsubscript{$\pm$0.85} & \cellcolor{lightgray}{\underline{61.42}}\textsubscript{$\pm$0.57} & \cellcolor{lightgray}{\underline{63.70}}\textsubscript{$\pm$1.23} & \cellcolor{lightgray}{\underline{59.76}}\textsubscript{$\pm$0.94} & \cellcolor{lightgray}{\underline{58.99}}\textsubscript{$\pm$0.69} \\
& & Embedding Fusion & 52.90\textsubscript{$\pm$1.09} & 38.56\textsubscript{$\pm$0.52} & 57.87\textsubscript{$\pm$0.98} & 26.95\textsubscript{$\pm$0.41} & 24.90\textsubscript{$\pm$1.21} & 40.24\textsubscript{$\pm$0.82} \\
& & \textbf{CBI} & \cellcolor{lightblue}{\textbf{66.53}}\textsubscript{$\pm$0.78} & \cellcolor{lightblue}{\textbf{59.40}}\textsubscript{$\pm$0.58} & \cellcolor{lightblue}{\textbf{64.51}}\textsubscript{$\pm$0.84} & \cellcolor{lightblue}{\textbf{83.35}}\textsubscript{$\pm$0.81} & \cellcolor{lightblue}{\textbf{69.47}}\textsubscript{$\pm$0.45} & \cellcolor{lightblue}{\textbf{68.65}}\textsubscript{$\pm$0.40} \\
\midrule
& \multirow{1}{*}{/} 
& \textcolor{black}{FFT~(upper bound)} & 81.53\textsubscript{$\pm$0.48} & 71.84\textsubscript{$\pm$0.62} & 75.83\textsubscript{$\pm$0.57} & 89.09\textsubscript{$\pm$0.39} & 75.83\textsubscript{$\pm$0.51} & 78.82\textsubscript{$\pm$0.46} \\ 
& \multirow{1}{*}{/}  & \textcolor{black}{ZST~(lower bound)} & 47.88\textsubscript{$\pm$1.73} & 23.18\textsubscript{$\pm$2.11} & 45.87\textsubscript{$\pm$1.84} & 25.36\textsubscript{$\pm$2.25} & 26.95\textsubscript{$\pm$1.96} & 33.85\textsubscript{$\pm$1.67} \\ 
\cmidrule(lr){2-9}
\multirow{2}{*}{\textbf{Phi-3.5-3.8B}} 
& \multirow{3}{*}{Adapter} 
& Text Concatenation & \cellcolor{lightgray}{\underline{59.35}}\textsubscript{$\pm$0.95} & \cellcolor{lightgray}{\underline{48.43}}\textsubscript{$\pm$1.27} & \cellcolor{lightgray}{\underline{60.96}}\textsubscript{$\pm$0.83} & \cellcolor{lightgray}{\underline{58.85}}\textsubscript{$\pm$0.71} & \cellcolor{lightgray}{\underline{60.41}}\textsubscript{$\pm$0.62} & \cellcolor{lightgray}{\underline{57.60}}\textsubscript{$\pm$0.86} \\
& & Embedding Fusion & 56.29\textsubscript{$\pm$0.77} & 40.13\textsubscript{$\pm$0.88} & 55.86\textsubscript{$\pm$1.09} & 31.20\textsubscript{$\pm$0.49} & 26.86\textsubscript{$\pm$1.21} & 42.07\textsubscript{$\pm$1.02} \\
& & \textbf{CBI} & \cellcolor{lightblue}{\textbf{61.21}}\textsubscript{$\pm$0.92} & \cellcolor{lightblue}{\textbf{59.09}}\textsubscript{$\pm$1.17} & \cellcolor{lightblue}{\textbf{63.89}}\textsubscript{$\pm$1.42} & \cellcolor{lightblue}{\textbf{71.60}}\textsubscript{$\pm$1.34} & \cellcolor{lightblue}{\textbf{67.59}}\textsubscript{$\pm$1.26} & \cellcolor{lightblue}{\textbf{64.68}}\textsubscript{$\pm$1.21} \\
\cmidrule(lr){2-9}
& \multirow{3}{*}{LoRA} 
& Text Concatenation & \cellcolor{lightgray}{\underline{61.13}}\textsubscript{$\pm$0.78} & \cellcolor{lightgray}{\underline{54.70}}\textsubscript{$\pm$0.94} & \cellcolor{lightgray}{\underline{61.73}}\textsubscript{$\pm$1.11} & \cellcolor{lightgray}{\underline{62.20}}\textsubscript{$\pm$0.65} & \cellcolor{lightgray}{\underline{63.43}}\textsubscript{$\pm$0.57} & \cellcolor{lightgray}{\underline{60.64}}\textsubscript{$\pm$0.76} \\
& & Embedding Fusion & 50.81\textsubscript{$\pm$0.93} & 27.90\textsubscript{$\pm$1.19} & 48.61\textsubscript{$\pm$0.72} & 27.40\textsubscript{$\pm$0.41} & 26.12\textsubscript{$\pm$1.24} & 36.17\textsubscript{$\pm$0.88} \\
& & \textbf{CBI} & \cellcolor{lightblue}{\textbf{70.00}}\textsubscript{$\pm$1.14} & \cellcolor{lightblue}{\textbf{59.25}}\textsubscript{$\pm$0.91} & \cellcolor{lightblue}{\textbf{70.22}}\textsubscript{$\pm$1.11} & \cellcolor{lightblue}{\textbf{84.25}}\textsubscript{$\pm$1.06} & \cellcolor{lightblue}{\textbf{72.90}}\textsubscript{$\pm$0.85} & \cellcolor{lightblue}{\textbf{71.32}}\textsubscript{$\pm$0.60} \\
\midrule
& \multirow{1}{*}{/} 
& \textcolor{black}{FFT~(upper bound)} 
& 72.89\textsubscript{$\pm$0.68} 
& 62.43\textsubscript{$\pm$0.79} 
& 72.97\textsubscript{$\pm$0.64} 
& 90.87\textsubscript{$\pm$0.88} 
& 65.89\textsubscript{$\pm$0.71} 
& 73.01\textsubscript{$\pm$0.83} 
\\
& \multirow{1}{*}{/}  
& \textcolor{black}{ZST~(lower bound)} 
& 40.02\textsubscript{$\pm$2.31} 
& 32.54\textsubscript{$\pm$2.47} 
& 39.27\textsubscript{$\pm$2.12} 
& 30.69\textsubscript{$\pm$2.76} 
& 19.73\textsubscript{$\pm$2.58} 
& 32.45\textsubscript{$\pm$2.34} 
\\ 

\cmidrule(lr){2-9}
\multirow{2}{*}{\textbf{Mistral-7B}} 
& \multirow{3}{*}{Adapter} 
& Text Concatenation & \cellcolor{lightgray}{\underline{58.31}}\textsubscript{$\pm$1.04} & \cellcolor{lightgray}{\underline{41.69}}\textsubscript{$\pm$0.87} & \cellcolor{lightgray}{\underline{59.72}}\textsubscript{$\pm$1.16} & \cellcolor{lightgray}{\underline{32.85}}\textsubscript{$\pm$0.59} & \cellcolor{lightgray}{\underline{31.59}}\textsubscript{$\pm$0.74} & \cellcolor{lightgray}{\underline{44.83}}\textsubscript{$\pm$0.93} \\
& & Embedding Fusion & 56.61\textsubscript{$\pm$0.79} & 37.46\textsubscript{$\pm$0.91} & 56.79\textsubscript{$\pm$0.98} & 30.15\textsubscript{$\pm$0.46} & 29.31\textsubscript{$\pm$1.08} & 42.06\textsubscript{$\pm$1.03} \\
& & \textbf{CBI} & \cellcolor{lightblue}{\textbf{64.60}}\textsubscript{$\pm$0.95} & \cellcolor{lightblue}{\textbf{60.97}}\textsubscript{$\pm$1.19} & \cellcolor{lightblue}{\textbf{64.20}}\textsubscript{$\pm$1.44} & \cellcolor{lightblue}{\textbf{77.90}}\textsubscript{$\pm$1.58} & \cellcolor{lightblue}{\textbf{70.53}}\textsubscript{$\pm$1.24} & \cellcolor{lightblue}{\textbf{67.64}}\textsubscript{$\pm$1.23} \\
\cmidrule(lr){2-9}
& \multirow{3}{*}{LoRA} 
& Text Concatenation & \cellcolor{lightgray}{\underline{57.58}}\textsubscript{$\pm$0.84} & \cellcolor{lightgray}{\underline{47.96}}\textsubscript{$\pm$1.11} & \cellcolor{lightgray}{\underline{61.11}}\textsubscript{$\pm$1.23} & \cellcolor{lightgray}{\underline{61.75}}\textsubscript{$\pm$0.63} & \cellcolor{lightgray}{\underline{62.94}}\textsubscript{$\pm$0.89} & \cellcolor{lightgray}{\underline{58.27}}\textsubscript{$\pm$0.73} \\
& & Embedding Fusion & 56.61\textsubscript{$\pm$0.77} & 37.46\textsubscript{$\pm$1.18} & 56.79\textsubscript{$\pm$0.85} & 30.15\textsubscript{$\pm$1.04} & 29.31\textsubscript{$\pm$0.91} & 42.06\textsubscript{$\pm$1.12} \\
& & \textbf{CBI} & \cellcolor{lightblue}{\textbf{68.31}}\textsubscript{$\pm$0.77} & \cellcolor{lightblue}{\textbf{58.78}}\textsubscript{$\pm$1.09} & \cellcolor{lightblue}{\textbf{70.52}}\textsubscript{$\pm$0.45} & \cellcolor{lightblue}{\textbf{86.30}}\textsubscript{$\pm$0.47} & \cellcolor{lightblue}{\textbf{75.43}}\textsubscript{$\pm$0.79} & \cellcolor{lightblue}{\textbf{71.87}}\textsubscript{$\pm$0.55} \\
\midrule
& \multirow{1}{*}{/} 
& \textcolor{black}{FFT~(upper bound)} 
& 74.93\textsubscript{$\pm$0.39} 
& 81.47\textsubscript{$\pm$0.31} 
& 77.38\textsubscript{$\pm$0.44} 
& 90.39\textsubscript{$\pm$0.36} 
& 80.79\textsubscript{$\pm$0.28} 
& 80.99\textsubscript{$\pm$0.39} \\ 
& \multirow{1}{*}{/}  
& \textcolor{black}{ZST~(lower bound)} 
& 43.27\textsubscript{$\pm$1.58} 
& 33.89\textsubscript{$\pm$1.91} 
& 41.03\textsubscript{$\pm$1.47} 
& 28.83\textsubscript{$\pm$2.03} 
& 21.47\textsubscript{$\pm$1.69} 
& 33.70\textsubscript{$\pm$1.61} \\ 
\cmidrule(lr){2-9}
\multirow{2}{*}{\textbf{Qwen2.5-7B}} 
& \multirow{3}{*}{Adapter} 
& Text Concatenation & \cellcolor{lightgray}{\underline{61.69}}\textsubscript{$\pm$0.92} & \cellcolor{lightgray}{\underline{47.65}}\textsubscript{$\pm$1.11} & \cellcolor{lightgray}{\underline{62.50}}\textsubscript{$\pm$0.58} & \cellcolor{lightgray}{\underline{51.35}}\textsubscript{$\pm$1.31} & \cellcolor{lightgray}{\underline{61.14}}\textsubscript{$\pm$0.44} & \cellcolor{lightgray}{\underline{56.87}}\textsubscript{$\pm$0.73} \\
& & Embedding Fusion & 58.31\textsubscript{$\pm$0.66} & 41.69\textsubscript{$\pm$1.27} & 59.72\textsubscript{$\pm$0.49} & 32.85\textsubscript{$\pm$0.95} & 31.59\textsubscript{$\pm$0.37} & 44.83\textsubscript{$\pm$1.12} \\
& & \textbf{CBI} & \cellcolor{lightblue}{\textbf{61.94}}\textsubscript{$\pm$1.08} & \cellcolor{lightblue}{\textbf{59.72}}\textsubscript{$\pm$1.21} & \cellcolor{lightblue}{\textbf{63.58}}\textsubscript{$\pm$1.43} & \cellcolor{lightblue}{\textbf{65.65}}\textsubscript{$\pm$1.54} & \cellcolor{lightblue}{\textbf{63.76}}\textsubscript{$\pm$1.12} & \cellcolor{lightblue}{\textbf{62.93}}\textsubscript{$\pm$1.27} \\
\cmidrule(lr){2-9}
& \multirow{3}{*}{LoRA} 
& Text Concatenation & \cellcolor{lightgray}{\underline{58.95}}\textsubscript{$\pm$0.73} & \cellcolor{lightgray}{\underline{55.80}}\textsubscript{$\pm$1.22} & \cellcolor{lightgray}{\underline{63.43}}\textsubscript{$\pm$0.47} & \cellcolor{lightgray}{\underline{65.25}}\textsubscript{$\pm$1.08} & \cellcolor{lightblue}{\textbf{64.08}}\textsubscript{$\pm$0.41} & \cellcolor{lightgray}{\underline{61.50}}\textsubscript{$\pm$0.95} \\
& & Embedding Fusion & 50.73\textsubscript{$\pm$1.19} & 36.05\textsubscript{$\pm$0.62} & 51.54\textsubscript{$\pm$0.88} & 27.30\textsubscript{$\pm$1.37} & 29.14\textsubscript{$\pm$0.52} & 38.95\textsubscript{$\pm$0.69} \\
& & \textbf{CBI} & \cellcolor{lightblue}{\textbf{62.74}}\textsubscript{$\pm$1.20} & \cellcolor{lightblue}{\textbf{59.72}}\textsubscript{$\pm$0.59} & \cellcolor{lightblue}{\textbf{65.12}}\textsubscript{$\pm$0.64} & \cellcolor{lightblue}{\textbf{76.70}}\textsubscript{$\pm$0.81} & \cellcolor{lightgray}{\underline{64.00}}\textsubscript{$\pm$0.68} & \cellcolor{lightblue}{\textbf{65.66}}\textsubscript{$\pm$0.83} \\
\midrule
& \multirow{1}{*}{/} 
& \textcolor{black}{FFT~(upper bound)} & 69.02\textsubscript{$\pm$0.42} & 72.79\textsubscript{$\pm$0.55} & 73.36\textsubscript{$\pm$0.63} & 91.71\textsubscript{$\pm$0.48} & 68.78\textsubscript{$\pm$0.51} & 75.13\textsubscript{$\pm$0.60} \\ 
& \multirow{1}{*}{/}  
& \textcolor{black}{ZST~(lower bound)} 
& 28.47\textsubscript{$\pm$1.72} 
& 27.81\textsubscript{$\pm$1.95} 
& 27.83\textsubscript{$\pm$2.10} 
& 30.65\textsubscript{$\pm$1.84} 
& 17.37\textsubscript{$\pm$2.25} 
& 26.43\textsubscript{$\pm$1.67} \\ 
\cmidrule(lr){2-9}
\multirow{2}{*}{\textbf{Llama-3.1-8B}} 
& \multirow{3}{*}{Adapter} 
& {Text Concatenation} & \cellcolor{lightgray}{\underline{61.69}}\textsubscript{$\pm$0.77} & \cellcolor{lightgray}{\underline{51.88}}\textsubscript{$\pm$1.36} & \cellcolor{lightblue}{\textbf{63.43}}\textsubscript{$\pm$0.59} & \cellcolor{lightgray}{\underline{62.30}}\textsubscript{$\pm$0.42} & \cellcolor{lightgray}{\underline{60.98}}\textsubscript{$\pm$1.18} & \cellcolor{lightgray}{\underline{60.06}}\textsubscript{$\pm$0.88} \\
& & Embedding Fusion & 57.02\textsubscript{$\pm$1.08} & 38.40\textsubscript{$\pm$0.64} & 59.57\textsubscript{$\pm$0.91} & 31.30\textsubscript{$\pm$0.55} & 27.18\textsubscript{$\pm$1.21} & 42.69\textsubscript{$\pm$0.73} \\
& & \textbf{CBI} & \cellcolor{lightblue}{\textbf{64.84}}\textsubscript{$\pm$0.88} & \cellcolor{lightblue}{\textbf{55.80}}\textsubscript{$\pm$1.14} & \cellcolor{lightgray}{\underline{63.43}}\textsubscript{$\pm$1.41} & \cellcolor{lightblue}{\textbf{89.20}}\textsubscript{$\pm$1.52} & \cellcolor{lightblue}{\textbf{72.98}}\textsubscript{$\pm$1.09} & \cellcolor{lightblue}{\textbf{69.25}}\textsubscript{$\pm$1.21} \\
\cmidrule(lr){2-9}
& \multirow{3}{*}{LoRA} 
& {Text Concatenation} & \cellcolor{lightgray}{\underline{65.40}}\textsubscript{$\pm$1.03} & \cellcolor{lightblue}{\textbf{58.93}}\textsubscript{$\pm$0.71} & \cellcolor{lightgray}{\underline{65.43}}\textsubscript{$\pm$1.19} & \cellcolor{lightgray}{\underline{71.90}}\textsubscript{$\pm$0.68} & \cellcolor{lightgray}{\underline{63.02}}\textsubscript{$\pm$1.34} & \cellcolor{lightgray}{\underline{64.94}}\textsubscript{$\pm$0.82} \\
& & Embedding Fusion & 54.52\textsubscript{$\pm$0.69} & 36.99\textsubscript{$\pm$1.24} & 57.10\textsubscript{$\pm$0.81} & 30.55\textsubscript{$\pm$1.41} & 28.90\textsubscript{$\pm$0.63} & 41.61\textsubscript{$\pm$0.97} \\
& & \textbf{CBI} & \cellcolor{lightblue}{\textbf{71.05}}\textsubscript{$\pm$0.36} & \cellcolor{lightgray}{\underline{58.78}}\textsubscript{$\pm$0.98} & \cellcolor{lightblue}{\textbf{69.44}}\textsubscript{$\pm$1.33} & \cellcolor{lightblue}{\textbf{89.35}}\textsubscript{$\pm$0.56} & \cellcolor{lightblue}{\textbf{77.88}}\textsubscript{$\pm$0.62} & \cellcolor{lightblue}{\textbf{73.30}}\textsubscript{$\pm$1.08} \\
\bottomrule
\end{tabular}
\caption{Performance comparison of Convolutional Bottleneck Interaction (CBI) with two baselines (concatenation of event embeddings and element-wise product of event embeddings)
under Adapter and LoRA PEFT settings evaluated on five temporal relation extraction datasets (micro-F1\%). 
Each score shows the mean\textsubscript{$\pm$std} of three runs. 
\textcolor{black}{Full fine-tuning (FFT) and zero-shot tuning (ZST) are reported as reference upper and lower bounds for PEFT methods, respectively.}
The best score in each setting (with the same backbone and PEFT method) is shown in \cellcolor{lightblue}\textbf{bold blue}, the second-best is shown in \cellcolor{lightgray}\underline{gray underlined}.}
\label{tab:main_results_tiny}
\end{table*}

\section{Experiments}
We conduct extensive empirical evaluations across diverse datasets and model families to demonstrate the utility of our framework. Moreover, we perform several experiments to investigate the effects of various design choices. 
\subsection{Benchmarks and Setup}
We evaluate the proposed {Convolutional Bottleneck Interaction (CBI)} on five widely used benchmarks for temporal relation extraction: 
These benchmarks differ in the granularity of temporal reasoning they require, some focus on intra-sentential event pairs, while others emphasize inter-sentential (discourse-level) relations that require long-range context understanding.
Detailed statistics of the datasets are shown in Appendix~\ref{sec:dataset_stats}. 

\paragraph{Intra-sentential benchmarks:}

\noindent \textbf{TimeBank-Dense (TB-Dense)}~\citep{cassidy-etal-2014-annotation}: An extension of the TimeBank corpus where all event pairs in each document (including across sentences) are annotated with temporal links. Although it contains both intra- and inter-sentential pairs, the majority are sentence-internal, making it primarily a local reasoning benchmark.

\noindent \textbf{MATRES}~\citep{ning-etal-2018-improving}: A dataset that refines TimeBank-Dense by filtering out uncertain annotations, focusing on reliably labeled event pairs in news articles. It uses three labels (\textit{BEFORE}, \textit{AFTER}, \textit{SIMULTANEOUS}). Most event pairs occur within the same sentence, emphasizing local temporal cues.

\noindent \textbf{TCR (Temporal and Causal Reasoning)}~\citep{ning-etal-2018-joint}: A dataset of sentences annotated for both temporal and causal relations between events. We use only the temporal relations (\textit{BEFORE}, \textit{AFTER}, \textit{SIMULTANEOUS}). This benchmark tests fine-grained intra-sentential reasoning due to the dense annotation of single-sentence event pairs.

\paragraph{Inter-sentential (discourse-level) benchmarks:}

\noindent \textbf{TDDiscourse-Man (TDD-Man)}~\citep{naik2019tddiscourse}: A discourse-level temporal event ordering dataset. It contains event pairs that are often far apart in the text (multiple sentences away), emphasizing long-range temporal reasoning across discourse units. The label set follows an interval relation scheme (\textit{BEFORE}, \textit{AFTER}, \textit{INCLUDES}, \textit{IS\_INCLUDED}, \textit{SIMULTANEOUS}).

\noindent \textbf{TDDiscourse-Auto (TDD-Auto)}~\citep{naik2019tddiscourse}: An automatically constructed large-scale dataset inferred from news articles using heuristics and external time annotations. It shares the same label set as TDD-Man but is much larger (tens of thousands of pairs) and noisier. Both TDD variants mainly involve inter-sentential temporal relations.

\section{Implementation Details}

We experiment with seven representative backbones that span both encoder-only and decoder-style large language models: \textbf{BERT-large}, \textbf{RoBERTa-large}, and five instruction-tuned LLMs, including \textbf{Gemma2-2B}, \textbf{Phi-3.5-3.8B}, \textbf{Mistral-7B}, \textbf{Qwen2.5-7B}, and \textbf{Llama-3.1-8B}. Our method is evaluated in two widely used PEFT paradigms, \textbf{Adapter} and \textbf{LoRA}, to ensure generality across architectures. 

\textcolor{black}{Each event pair is represented by two contextual segments marked as 
[event1]...[/event1] and [event2]...[/event2], which are jointly fed into the model to predict their temporal relation.
For classification, we use the hidden state of the last non-padding token. Hidden states are projected to logits via a linear layer, and the logits at the last valid position are used for training and inference.}

In all experiments, we fine-tune each model on each dataset independently. Encoder-based models are trained with a batch size of 12 for up to 10 epochs, while LLMs, which are more computationally demanding, use a batch size of 8 under the same epoch limit. In both cases, we employ the AdamW optimizer with a learning rate of $1\times10^{-4}$ and a dropout rate of 0.1. The best checkpoint is selected according to validation micro-F1. Class imbalance is handled by reporting micro-F1 and, where appropriate, oversampling minority classes during training. 
The kernel size $k$ is set to 5.
The hidden-state dimensionality follows the default configuration of each backbone model.
Following~\citet{qu2024introducing}, the rank of bottleneck $r$ is fixed at 128 for both the adapter and LoRA settings.
All experiments are conducted on NVIDIA A100 GPU under a unified training protocol to ensure fair comparison.

\textcolor{black}{For full fine-tuning, we use a learning rate of $2\times10^{-5}$ with a batch size of 2 by following~\citet{pareja2024unveiling,kim2025rethinking}.
Training is conducted on two NVIDIA A100 GPUs using Fully Sharded Data Parallel (FSDP) with the {FULL\_SHARD} strategy. 
To reduce memory consumption, we enable bfloat16 (bf16) precision during training.}

\textcolor{black}{For zero-shot tuning, we use the following prompt:
"You are an expert in temporal information extraction. Determine the temporal relation between the two events based on their contexts. Possible relations: \{label set\}. Reply with ONLY one label from: \{input text\}. No explanation.".
The \{input text\} is:
"Event 1 context: \{text1\}. Event 2 context: \{text2\}. Temporal relation:" and "\{label set\}" is a list all labels.}

\subsection{Comparison with state-of-the-art Methods}

The experiments are designed to assess whether CBI generalizes under varying backbone capacities and PEFT configurations.
We compare CBI with two strong baselines: \textbf{Text Concatenation}~\citep{zhou-etal-2022-rsgt,roccabruna2024will}, a non-interactive approach that concatenates the two event texts into a single input sequence, encodes them into a joint representation, and then feeds this representation into a low-rank bottleneck, and \textbf{Embedding Fusion}~\citep{chaturvedi-etal-2025-temporal}, which performs event fusion at the input level by taking the element-wise product of the two event embeddings and using the fused embedding as the model input.\footnote{In the original design, the two event embeddings and their element-wise product are concatenated along the hidden dimension. In the PEFT setting, however, the backbone hidden size is fixed, making this three-way concatenation infeasible; following \citet{chaturvedi-etal-2025-temporal}, we therefore retain the element-wise product, which is the core fusion component.}
In contrast to Embedding Fusion, which completes event fusion before the bottleneck and propagates a single fused embedding, CBI places the interaction inside the low-rank bottleneck: the two events remain separate at the input, and the limited bottleneck capacity is explicitly devoted to modeling cross-event interactions, thereby strengthening event interaction precisely in a space where such dependencies are otherwise hard to learn.
The experimental results for five datasets and seven backbones are summarized in Table~\ref{tab:main_results_tiny}.

Across all configurations, CBI consistently achieves the highest micro-F1, demonstrating robust temporal relation extraction in PEFT settings.
In the Adapter setting, CBI improves RoBERTa-large from 44.83\% (Text Concatenation) and 49.94\% (Embedding Fusion) to 76.59\% (+31.76 percentage points, abbreviated as pp / +26.65pp), and BERT-large from 44.83\% / 50.19\% to 76.28\% (+31.45pp / +26.09pp).
A similar trend holds for LLMs, where CBI achieves gains of +11.44pp on Gemma2-2B, +7.08pp on Phi-3.5-3.8B, +22.81pp on Mistral-7B, +6.06pp on Qwen2.5-7B, and +9.19pp on Llama-3.1-8B compared to Text Concatenation.
Notably, Embedding Fusion provides limited or even negative gains on larger models, indicating that multiplication interaction of the input embeddings of the respective events is insufficient to capture complex temporal dependencies.
Under the LoRA configuration, CBI maintains clear superiority, improving over Text Concatenation by an average of +6.36pp, with the largest margin (+13.6pp) observed on Mistral-7B.

\textcolor{black}{
It is clear from Table~\ref{tab:temporal_routing_results_combined_full} that, within the same backbone, all PEFT methods achieve higher micro-F1 than zero-shot tuning.\footnote{Few-shots~(5-shots) learning performs slightly worse than zero-shot and exhibits higher variance, so we use zero-shot tuning performance as the lower bound for the PEFT method.} 
For example, instruction-tuned models such as Mistral-7B and Llama 3.1 8B are around 25 to 35 micro-F1 score in zero-shot, while PEFT results are mostly around 60 to 70 micro-F1 score. 
This shows that PEFT is necessary and zero-shot defines a clear lower bound.
However, PEFT is still worse than full fine-tuning, which defines the upper bound.\footnote{We applied full fine-tuning to both Text Concatenation and Embedding Fusion and found that Text Concatenation largely outperformed Embedding Fusion.
}
On Llama-3.1-8B, the best PEFT result is 73.30 micro-F1 score with LoRA and CBI. This is still below FFT performance (75.13 micro-F1 score), which indicates a slightly reduced representational capacity of the PEFT models. 
CBI consistently achieves the best results in the PEFT setting. Under Adapter, CBI improves average micro-F1 score from 60.06 (Text Concatenation) to 69.25. Under LoRA, it improves from 64.94 to 73.30 micro-F1 score. Similar gains are observed across datasets. This shows that CBI recovers more of the lost capacity.
}

\subsection{Effectiveness of CBI}

The core principle in our framework is to model the interaction in the feature (i.e., latent) space between two event representations using an element-wise mathematical operation. Therefore, we compare the effect of three different possible element-wise operations: Multiplication, Addition and Subtraction. 
Multiplication captures token co-activation but can amplify noise; Addition smooths the interaction but may blur important distinctions; Subtraction highlights contrasts but lacks cues for temporal order.\footnote{We take the absolute difference to avoid a bias caused by the textual position of events.} Note that none of these proposed variants perform the convolution operation. Consequently, Multiplication represents an ablated version of CBI.

Table~\ref{tab:temporal_routing_results_combined_full} shows the results of all three variants, as well as CBI for the adapters and LoRA settings in all backbone architectures and datasets. 
Although all four interaction functions are our designs, they reflect different ways events may interact through the low-rank bottlenecks. Empirically, CBI consistently outperforms these simpler operations. For example, on RoBERTa-large (Adapter), Multiplication reaches 73.35\% F1 and Subtraction 62.87\% F1, while CBI achieves 76.59\% F1. 
A similar pattern holds for Qwen-2.5-7B (LoRA), where CBI obtains 65.66\% F1 compared to 59.64\% for Multiplication and 61.53\% for Subtraction, but with much better stability and transferability between datasets. 
Overall, CBI preserves the efficiency of simple interaction functions while adding structure awareness through local depthwise convolutions and improving local coherence by filtering and aligning temporally relevant cues before cross-event interaction. This leads to more consistent and interpretable representations across architectures.

\begin{figure}[htbp]
\centering
\subfloat[\scriptsize{Adapter-based tuning.}]{\includegraphics[width=0.4\textwidth]{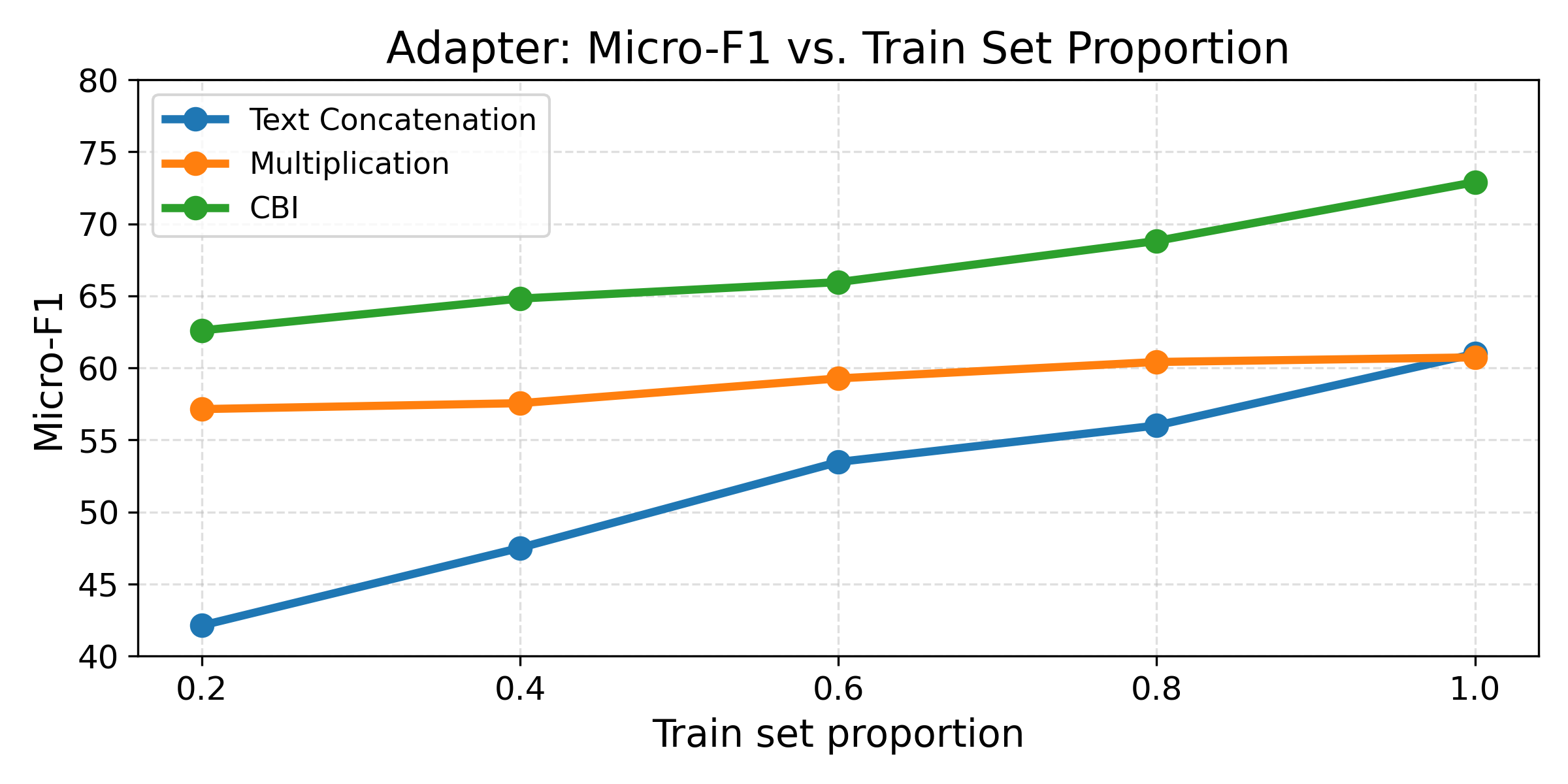}
\label{fig:adapter_curve}
}
\hfil
\subfloat[\scriptsize{LoRA-based tuning.}]{\includegraphics[width=0.4\textwidth]{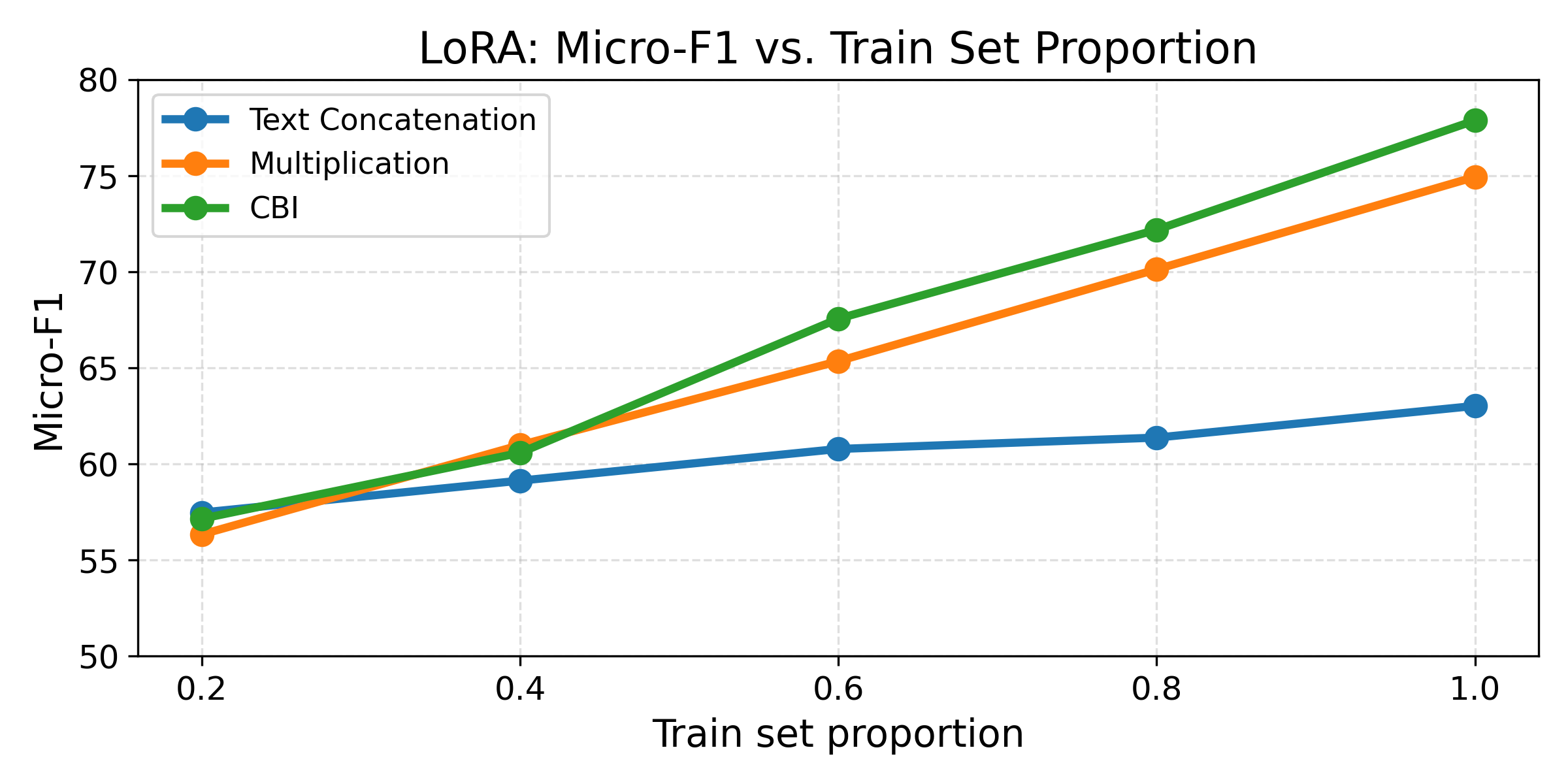}
\label{fig:lora_curve}
}
\caption{Effect of training data proportion on model performance. Both CBI and Multiplication are methods proposed in this paper, while Text Concatenation serves as the baseline.}
\label{fig:adapter_lora_curve}
\end{figure}

\begin{table*}[t!]
    \centering
    \tiny
    \renewcommand{\arraystretch}{1.15}
    \setlength{\tabcolsep}{9pt}
    
    \begin{tabular}{ll|l|ccccc|c}
        \toprule
        \textbf{Backbone} & \textbf{PEFT} & \textbf{Routing} &
        \textbf{Matres} & \textbf{TB-dense} & \textbf{TCR} & \textbf{TDD-auto} & \textbf{TDD-man} & \textbf{AVG} \\
        \midrule
        \multirow{8}{*}{\textbf{Roberta-large}}
        & Adapter & Multiplication & \cellcolor{lightgray}\underline{71.61\textsubscript{$\pm$0.78}} & \cellcolor{lightgray}\underline{60.82\textsubscript{$\pm$1.34}} & \cellcolor{lightgray}\underline{68.83\textsubscript{$\pm$1.22}} & \cellcolor{lightblue}\textbf{89.15\textsubscript{$\pm$1.49}} & \cellcolor{lightgray}\underline{76.33\textsubscript{$\pm$0.91}} & \cellcolor{lightgray}\underline{73.35\textsubscript{$\pm$1.18}} \\
        & Adapter & Addition     & 58.79\textsubscript{$\pm$1.55} & 41.69\textsubscript{$\pm$0.92} & 59.72\textsubscript{$\pm$1.36} & 56.35\textsubscript{$\pm$1.21} & 34.04\textsubscript{$\pm$1.44} & 50.12\textsubscript{$\pm$1.47} \\
        & Adapter & Subtraction        & 58.31\textsubscript{$\pm$1.09} & 41.69\textsubscript{$\pm$1.23} & 59.72\textsubscript{$\pm$1.28} & 85.40\textsubscript{$\pm$1.63} & 69.22\textsubscript{$\pm$0.98} & 62.87\textsubscript{$\pm$1.25} \\
        & Adapter & \textbf{CBI} & \cellcolor{lightblue}\textbf{74.35\textsubscript{$\pm$0.85}} & \cellcolor{lightblue}\textbf{68.18\textsubscript{$\pm$1.07}} & \cellcolor{lightblue}\textbf{73.92\textsubscript{$\pm$1.61}} & \cellcolor{lightgray}\underline{86.85\textsubscript{$\pm$1.32}} & \cellcolor{lightblue}\textbf{79.67\textsubscript{$\pm$1.18}} & \cellcolor{lightblue}\textbf{76.59\textsubscript{$\pm$1.22}} \\
        \cmidrule(lr){2-9}
        & LoRA   & Multiplication & 68.79\textsubscript{$\pm$1.31} & 60.34\textsubscript{$\pm$0.52} & \cellcolor{lightblue}\textbf{70.06\textsubscript{$\pm$0.95}} & \cellcolor{lightgray}\underline{87.40}\textsubscript{$\pm$0.83} & 71.02\textsubscript{$\pm$0.59} & \cellcolor{lightgray}\underline{71.52\textsubscript{$\pm$0.90}} \\
        & LoRA   & Addition     & 66.45\textsubscript{$\pm$0.38} & \cellcolor{lightgray}\underline{61.60\textsubscript{$\pm$0.83}} & 65.43\textsubscript{$\pm$1.33} & 78.65\textsubscript{$\pm$0.68} & \cellcolor{lightgray}\underline{73.71\textsubscript{$\pm$1.08}} & 69.17\textsubscript{$\pm$0.99} \\
        & LoRA   & Subtraction        & \cellcolor{lightblue}\textbf{70.56\textsubscript{$\pm$1.22}} & 52.82\textsubscript{$\pm$1.10} & \cellcolor{lightgray}\underline{68.21}\textsubscript{$\pm$0.63} & 80.00\textsubscript{$\pm$0.72} & 68.00\textsubscript{$\pm$0.60} & 67.92\textsubscript{$\pm$0.52} \\
        & LoRA   & \textbf{CBI} & \cellcolor{lightgray}\underline{70.16\textsubscript{$\pm$1.38}} & \cellcolor{lightblue}\textbf{62.70\textsubscript{$\pm$0.74}} & 67.75\textsubscript{$\pm$1.20} & \cellcolor{lightblue}\textbf{89.35\textsubscript{$\pm$0.83}} & \cellcolor{lightblue}\textbf{76.24\textsubscript{$\pm$1.27}} & \cellcolor{lightblue}\textbf{73.24\textsubscript{$\pm$1.00}} \\
        \midrule
        \multirow{8}{*}{\textbf{Bert-large}}
        & Adapter & Multiplication & \cellcolor{lightblue}\textbf{74.35\textsubscript{$\pm$0.83}} & \cellcolor{lightgray}\underline{59.56}\textsubscript{$\pm$1.31} & \cellcolor{lightgray}\underline{68.21}\textsubscript{$\pm$1.25} & \cellcolor{lightgray}\underline{87.25}\textsubscript{$\pm$1.48} & \cellcolor{lightgray}\underline{75.84\textsubscript{$\pm$1.02}} & \cellcolor{lightgray}\underline{73.04\textsubscript{$\pm$1.17}} \\
        & Adapter & Addition     & 58.31\textsubscript{$\pm$1.26} & 46.87\textsubscript{$\pm$1.12} & 59.72\textsubscript{$\pm$1.09} & 73.60\textsubscript{$\pm$1.34} & 60.00\textsubscript{$\pm$1.72} & 59.70\textsubscript{$\pm$1.33} \\
        & Adapter & Subtraction        & 58.31\textsubscript{$\pm$1.05} & 53.76\textsubscript{$\pm$1.42} & 59.72\textsubscript{$\pm$1.14} & 83.05\textsubscript{$\pm$1.69} & 72.49\textsubscript{$\pm$1.21} & 65.47\textsubscript{$\pm$1.36} \\
        & Adapter & \textbf{CBI} & \cellcolor{lightgray}\underline{71.69}\textsubscript{$\pm$0.97} & \cellcolor{lightblue}\textbf{67.55\textsubscript{$\pm$1.08}} & \cellcolor{lightblue}\textbf{72.53\textsubscript{$\pm$1.51}} & \cellcolor{lightblue}\textbf{90.60\textsubscript{$\pm$1.44}} & \cellcolor{lightblue}\textbf{79.02\textsubscript{$\pm$1.19}} & \cellcolor{lightblue}\textbf{76.28\textsubscript{$\pm$1.21}} \\
        \cmidrule(lr){2-9}
        & LoRA   & Multiplication & 67.34\textsubscript{$\pm$1.25} & \cellcolor{lightgray}\underline{61.13}\textsubscript{$\pm$0.99} & \cellcolor{lightblue}\textbf{69.14\textsubscript{$\pm$0.62}} & \cellcolor{lightgray}\underline{87.95\textsubscript{$\pm$1.13}} & \cellcolor{lightgray}\underline{74.53\textsubscript{$\pm$1.11}} & \cellcolor{lightgray}\underline{72.02\textsubscript{$\pm$0.40}} \\
        & LoRA   & Addition     & 67.42\textsubscript{$\pm$0.60} & 57.52\textsubscript{$\pm$1.19} & 65.59\textsubscript{$\pm$0.98} & 80.25\textsubscript{$\pm$1.14} & 71.76\textsubscript{$\pm$0.46} & 68.51\textsubscript{$\pm$0.65} \\
        & LoRA   & Subtraction        & \cellcolor{lightgray}\underline{70.56\textsubscript{$\pm$1.04}} & 52.98\textsubscript{$\pm$0.57} & 66.82\textsubscript{$\pm$1.09} & 83.85\textsubscript{$\pm$1.09} & 71.35\textsubscript{$\pm$0.64} & 69.11\textsubscript{$\pm$1.36} \\
        & LoRA   & \textbf{CBI} & \cellcolor{lightblue}\textbf{71.94\textsubscript{$\pm$1.05}} & \cellcolor{lightblue}\textbf{65.36\textsubscript{$\pm$0.89}} & \cellcolor{lightgray}\underline{68.21}\textsubscript{$\pm$0.62} & \cellcolor{lightblue}\textbf{88.25\textsubscript{$\pm$1.19}} & \cellcolor{lightblue}\textbf{76.82\textsubscript{$\pm$0.53}} & \cellcolor{lightblue}\textbf{74.12\textsubscript{$\pm$0.76}} \\
        \midrule
        \multirow{8}{*}{\textbf{Gemma2 2B}}
        & Adapter & Multiplication & 58.87\textsubscript{$\pm$1.34} & 53.61\textsubscript{$\pm$1.17} & 63.12\textsubscript{$\pm$1.22} & \cellcolor{lightgray}\underline{63.60\textsubscript{$\pm$1.31}} & 62.29\textsubscript{$\pm$1.09} & 60.30\textsubscript{$\pm$1.37} \\
        & Adapter & Addition     & 59.52\textsubscript{$\pm$1.21} & 57.99\textsubscript{$\pm$1.47} & \cellcolor{lightgray}\underline{63.89}\textsubscript{$\pm$1.33} & 60.90\textsubscript{$\pm$1.28} & 65.31\textsubscript{$\pm$1.52} & 61.52\textsubscript{$\pm$1.44} \\
        & Adapter & Subtraction        & \cellcolor{lightgray}\underline{61.13\textsubscript{$\pm$0.89}} & \cellcolor{lightgray}\underline{58.31\textsubscript{$\pm$1.63}} & 63.58\textsubscript{$\pm$1.17} & 59.55\textsubscript{$\pm$1.26} & \cellcolor{lightgray}\underline{65.39\textsubscript{$\pm$1.31}} & \cellcolor{lightgray}\underline{61.59\textsubscript{$\pm$1.35}} \\
        & Adapter & \textbf{CBI} & \cellcolor{lightblue}\textbf{62.58\textsubscript{$\pm$1.02}} & \cellcolor{lightblue}\textbf{59.09\textsubscript{$\pm$1.14}} & \cellcolor{lightblue}\textbf{64.20\textsubscript{$\pm$1.48}} & \cellcolor{lightblue}\textbf{77.20\textsubscript{$\pm$1.55}} & \cellcolor{lightblue}\textbf{68.24\textsubscript{$\pm$1.27}} & \cellcolor{lightblue}\textbf{66.26\textsubscript{$\pm$1.29}} \\
        \cmidrule(lr){2-9}
        & LoRA   & Multiplication & \cellcolor{lightgray}\underline{63.55\textsubscript{$\pm$0.60}} & 55.80\textsubscript{$\pm$0.69} & 64.35\textsubscript{$\pm$0.93} & \cellcolor{lightgray}\underline{67.95\textsubscript{$\pm$0.70}} & \cellcolor{lightgray}\underline{65.55\textsubscript{$\pm$0.58}} & \cellcolor{lightgray}\underline{63.44\textsubscript{$\pm$1.37}} \\
        & LoRA   & Addition     & 61.53\textsubscript{$\pm$1.30} & 56.43\textsubscript{$\pm$1.10} & 63.58\textsubscript{$\pm$1.09} & 59.50\textsubscript{$\pm$0.90} & 62.37\textsubscript{$\pm$1.34} & 60.68\textsubscript{$\pm$0.85} \\
        & LoRA   & Subtraction        & 61.45\textsubscript{$\pm$0.64} & \cellcolor{lightgray}\underline{57.68\textsubscript{$\pm$1.14}} & \cellcolor{lightblue}\textbf{64.81\textsubscript{$\pm$0.97}} & 60.15\textsubscript{$\pm$0.55} & 61.71\textsubscript{$\pm$0.98} & 61.16\textsubscript{$\pm$1.32} \\
        & LoRA   & \textbf{CBI} & \cellcolor{lightblue}\textbf{66.53\textsubscript{$\pm$0.78}} & \cellcolor{lightblue}\textbf{59.40\textsubscript{$\pm$0.58}} & \cellcolor{lightgray}\underline{64.51\textsubscript{$\pm$0.84}} & \cellcolor{lightblue}\textbf{83.35\textsubscript{$\pm$0.81}} & \cellcolor{lightblue}\textbf{69.47\textsubscript{$\pm$0.45}} & \cellcolor{lightblue}\textbf{68.65\textsubscript{$\pm$0.40}} \\
        \midrule
        \multirow{8}{*}{\textbf{Phi-3.5 3.8B}}
        & Adapter & Multiplication & \cellcolor{lightblue}\textbf{63.31\textsubscript{$\pm$1.19}} & 58.31\textsubscript{$\pm$1.22} & 63.58\textsubscript{$\pm$1.07} & \cellcolor{lightgray}\underline{61.55\textsubscript{$\pm$1.45}} & 61.96\textsubscript{$\pm$1.12} & 61.74\textsubscript{$\pm$1.33} \\
        & Adapter & Addition     & \cellcolor{lightgray}\underline{62.82}\textsubscript{$\pm$1.07} & 58.46\textsubscript{$\pm$1.28} & \cellcolor{lightgray}\underline{64.33}\textsubscript{$\pm$1.21} & 61.50\textsubscript{$\pm$1.56} & \cellcolor{lightgray}\underline{63.67\textsubscript{$\pm$1.09}} & \cellcolor{lightgray}\underline{62.25\textsubscript{$\pm$1.27}} \\
        & Adapter & Subtraction        & 62.74\textsubscript{$\pm$1.44} & \cellcolor{lightgray}\underline{58.62\textsubscript{$\pm$1.03}} & \cellcolor{lightblue}\textbf{64.81\textsubscript{$\pm$1.35}} & 60.45\textsubscript{$\pm$1.72} & 63.35\textsubscript{$\pm$1.18} & 61.99\textsubscript{$\pm$1.39} \\
        & Adapter & \textbf{CBI} & 61.21\textsubscript{$\pm$0.92} & \cellcolor{lightblue}\textbf{59.09\textsubscript{$\pm$1.17}} & 63.89\textsubscript{$\pm$1.42} & \cellcolor{lightblue}\textbf{71.60\textsubscript{$\pm$1.34}} & \cellcolor{lightblue}\textbf{67.59\textsubscript{$\pm$1.26}} & \cellcolor{lightblue}\textbf{64.68\textsubscript{$\pm$1.21}} \\
        \cmidrule(lr){2-9}
        & LoRA   & Multiplication & 62.34\textsubscript{$\pm$1.43} & 54.86\textsubscript{$\pm$0.99} & 64.04\textsubscript{$\pm$0.88} & \cellcolor{lightgray}\underline{64.65\textsubscript{$\pm$1.13}} & \cellcolor{lightgray}\underline{64.08\textsubscript{$\pm$0.36}} & \cellcolor{lightgray}\underline{61.99\textsubscript{$\pm$1.13}} \\
        & LoRA   & Addition     & 60.24\textsubscript{$\pm$0.63} & 57.21\textsubscript{$\pm$0.84} & 64.81\textsubscript{$\pm$0.85} & 60.30\textsubscript{$\pm$1.06} & 60.98\textsubscript{$\pm$0.89} & 60.71\textsubscript{$\pm$0.64} \\
        & LoRA   & Subtraction        & \cellcolor{lightgray}\underline{62.66\textsubscript{$\pm$0.58}} & \cellcolor{lightblue}\textbf{60.50\textsubscript{$\pm$0.96}} & \cellcolor{lightgray}\underline{65.59\textsubscript{$\pm$0.76}} & 59.70\textsubscript{$\pm$1.18} & 60.57\textsubscript{$\pm$0.77} & 61.80\textsubscript{$\pm$0.61} \\
        & LoRA   & \textbf{CBI} & \cellcolor{lightblue}\textbf{70.00\textsubscript{$\pm$1.14}} & \cellcolor{lightgray}\underline{59.25\textsubscript{$\pm$0.91}} & \cellcolor{lightblue}\textbf{70.22\textsubscript{$\pm$1.11}} & \cellcolor{lightblue}\textbf{84.25\textsubscript{$\pm$1.06}} & \cellcolor{lightblue}\textbf{72.90\textsubscript{$\pm$0.85}} & \cellcolor{lightblue}\textbf{71.32\textsubscript{$\pm$0.60}} \\
        \midrule
        \multirow{8}{*}{\textbf{Mistral-7B}}
        & Adapter & Multiplication & \cellcolor{lightgray}\underline{63.87\textsubscript{$\pm$1.22}} & \cellcolor{lightgray}\underline{54.23\textsubscript{$\pm$1.38}} & \cellcolor{lightgray}\underline{61.88\textsubscript{$\pm$1.29}} & \cellcolor{lightgray}\underline{62.30\textsubscript{$\pm$1.51}} & \cellcolor{lightgray}\underline{61.63\textsubscript{$\pm$1.11}} & \cellcolor{lightgray}\underline{60.78\textsubscript{$\pm$1.28}} \\
        & Adapter & Addition     & 58.31\textsubscript{$\pm$1.18} & 41.69\textsubscript{$\pm$1.27} & 59.72\textsubscript{$\pm$1.33} & 32.85\textsubscript{$\pm$1.47} & 46.12\textsubscript{$\pm$1.36} & 47.74\textsubscript{$\pm$1.42} \\
        & Adapter & Subtraction        & 58.31\textsubscript{$\pm$1.06} & 41.69\textsubscript{$\pm$1.41} & 59.72\textsubscript{$\pm$1.21} & 37.20\textsubscript{$\pm$1.68} & 31.59\textsubscript{$\pm$1.09} & 45.70\textsubscript{$\pm$1.37} \\
        & Adapter & \textbf{CBI} & \cellcolor{lightblue}\textbf{64.60\textsubscript{$\pm$0.95}} & \cellcolor{lightblue}\textbf{60.97\textsubscript{$\pm$1.19}} & \cellcolor{lightblue}\textbf{64.20\textsubscript{$\pm$1.44}} & \cellcolor{lightblue}\textbf{77.90\textsubscript{$\pm$1.58}} & \cellcolor{lightblue}\textbf{70.53\textsubscript{$\pm$1.24}} & \cellcolor{lightblue}\textbf{67.64\textsubscript{$\pm$1.23}} \\
        \cmidrule(lr){2-9}
        & LoRA   & Multiplication & \cellcolor{lightgray}\underline{68.06\textsubscript{$\pm$0.45}} & \cellcolor{lightblue}\textbf{60.19\textsubscript{$\pm$1.07}} & \cellcolor{lightgray}\underline{67.75\textsubscript{$\pm$1.22}} & \cellcolor{lightgray}\underline{82.25\textsubscript{$\pm$0.95}} & \cellcolor{lightgray}\underline{72.24\textsubscript{$\pm$1.19}} & \cellcolor{lightgray}\underline{70.10\textsubscript{$\pm$0.90}} \\
        & LoRA   & Addition     & 62.74\textsubscript{$\pm$0.88} & 56.43\textsubscript{$\pm$1.17} & 64.97\textsubscript{$\pm$0.45} & 59.60\textsubscript{$\pm$0.61} & 60.49\textsubscript{$\pm$0.64} & 60.85\textsubscript{$\pm$1.06} \\
        & LoRA   & Subtraction        & 63.95\textsubscript{$\pm$0.55} & 56.90\textsubscript{$\pm$0.92} & 64.35\textsubscript{$\pm$0.40} & 60.20\textsubscript{$\pm$0.99} & 59.18\textsubscript{$\pm$0.77} & 60.92\textsubscript{$\pm$0.85} \\
        & LoRA   & \textbf{CBI} & \cellcolor{lightblue}\textbf{68.31\textsubscript{$\pm$0.77}} & \cellcolor{lightgray}\underline{58.78\textsubscript{$\pm$1.09}} & \cellcolor{lightblue}\textbf{70.52\textsubscript{$\pm$0.45}} & \cellcolor{lightblue}\textbf{86.30\textsubscript{$\pm$0.47}} & \cellcolor{lightblue}\textbf{75.43\textsubscript{$\pm$0.79}} & \cellcolor{lightblue}\textbf{71.87\textsubscript{$\pm$0.55}} \\
        \midrule
        \multirow{8}{*}{\textbf{Qwen2.5 7B}}
        & Adapter & Multiplication & 60.08\textsubscript{$\pm$1.11} & 53.29\textsubscript{$\pm$1.33} & 61.88\textsubscript{$\pm$1.24} & \cellcolor{lightgray}\underline{60.40}\textsubscript{$\pm$1.42} & 60.73\textsubscript{$\pm$1.09} & 59.28\textsubscript{$\pm$1.25} \\
        & Adapter & Addition     & \cellcolor{lightblue}\textbf{63.47\textsubscript{$\pm$0.97}} & 57.52\textsubscript{$\pm$1.26} & \cellcolor{lightblue}\textbf{64.20\textsubscript{$\pm$1.35}} & 60.25\textsubscript{$\pm$1.31} & \cellcolor{lightgray}\underline{62.45\textsubscript{$\pm$1.41}} & 61.58\textsubscript{$\pm$1.23} \\
        & Adapter & Subtraction        & \cellcolor{lightgray}\underline{63.31\textsubscript{$\pm$1.15}} & \cellcolor{lightgray}\underline{58.15\textsubscript{$\pm$1.18}} & \cellcolor{lightgray}\underline{64.18\textsubscript{$\pm$1.17}} & 60.20\textsubscript{$\pm$1.37} & 62.20\textsubscript{$\pm$1.28} & \cellcolor{lightgray}\underline{61.61\textsubscript{$\pm$1.29}} \\
        & Adapter & \textbf{CBI} & 61.94\textsubscript{$\pm$1.08} & \cellcolor{lightblue}\textbf{59.72\textsubscript{$\pm$1.21}} & 63.58\textsubscript{$\pm$1.43} & \cellcolor{lightblue}\textbf{65.65\textsubscript{$\pm$1.54}} & \cellcolor{lightblue}\textbf{63.76\textsubscript{$\pm$1.12}} & \cellcolor{lightblue}\textbf{62.93\textsubscript{$\pm$1.27}} \\
        \cmidrule(lr){2-9}
        & LoRA   & Multiplication & 60.16\textsubscript{$\pm$1.22} & 56.58\textsubscript{$\pm$0.58} & 61.42\textsubscript{$\pm$0.60} & \cellcolor{lightgray}\underline{60.30\textsubscript{$\pm$0.53}} & 59.76\textsubscript{$\pm$0.48} & 59.64\textsubscript{$\pm$0.67} \\
        & LoRA   & Addition     & 62.26\textsubscript{$\pm$0.88} & 57.05\textsubscript{$\pm$0.62} & \cellcolor{lightgray}\underline{63.12}\textsubscript{$\pm$0.82} & 56.90\textsubscript{$\pm$0.82} & \cellcolor{lightgray}\underline{62.94\textsubscript{$\pm$0.52}} & 60.45\textsubscript{$\pm$0.79} \\
        & LoRA   & Subtraction        & \cellcolor{lightblue}\textbf{63.87\textsubscript{$\pm$0.99}} & \cellcolor{lightgray}\underline{58.46\textsubscript{$\pm$0.65}} & 62.81\textsubscript{$\pm$0.90} & 60.05\textsubscript{$\pm$1.34} & 62.45\textsubscript{$\pm$0.92} & \cellcolor{lightgray}\underline{61.53\textsubscript{$\pm$0.73}} \\
        & LoRA   & \textbf{CBI} & \cellcolor{lightgray}\underline{62.74}\textsubscript{$\pm$1.20} & \cellcolor{lightblue}\textbf{59.72\textsubscript{$\pm$0.59}} & \cellcolor{lightblue}\textbf{65.12\textsubscript{$\pm$0.64}} & \cellcolor{lightblue}\textbf{76.70\textsubscript{$\pm$0.81}} & \cellcolor{lightblue}\textbf{64.00\textsubscript{$\pm$0.68}} & \cellcolor{lightblue}\textbf{65.66\textsubscript{$\pm$0.83}} \\
        \midrule
        \multirow{8}{*}{\textbf{Llama 3.1 8B}}
        & Adapter & Multiplication & 62.50\textsubscript{$\pm$0.91} & 56.11\textsubscript{$\pm$1.28} & 61.27\textsubscript{$\pm$1.32} & \cellcolor{lightgray}\underline{70.40\textsubscript{$\pm$1.44}} & 60.73\textsubscript{$\pm$1.07} & \cellcolor{lightgray}\underline{62.20\textsubscript{$\pm$1.19}} \\
        & Adapter & Addition     & 62.90\textsubscript{$\pm$1.34} & \cellcolor{lightblue}\textbf{57.21\textsubscript{$\pm$1.23}} & \cellcolor{lightblue}\textbf{65.12\textsubscript{$\pm$1.25}} & 60.80\textsubscript{$\pm$1.57} & \cellcolor{lightgray}\underline{62.69\textsubscript{$\pm$1.16}} & 61.74\textsubscript{$\pm$1.26} \\
        & Adapter & Subtraction        & \cellcolor{lightgray}\underline{63.71\textsubscript{$\pm$1.07}} & \cellcolor{lightgray}\underline{56.27}\textsubscript{$\pm$1.36} & \cellcolor{lightgray}\underline{64.97\textsubscript{$\pm$1.18}} & 60.90\textsubscript{$\pm$1.61} & 62.29\textsubscript{$\pm$1.24} & 61.63\textsubscript{$\pm$1.22} \\
        & Adapter & \textbf{CBI} & \cellcolor{lightblue}\textbf{64.84\textsubscript{$\pm$0.88}} & 55.80\textsubscript{$\pm$1.14} & 63.43\textsubscript{$\pm$1.41} & \cellcolor{lightblue}\textbf{88.20\textsubscript{$\pm$1.52}} & \cellcolor{lightblue}\textbf{72.89\textsubscript{$\pm$1.09}} & \cellcolor{lightblue}\textbf{69.25\textsubscript{$\pm$1.21}} \\
        \cmidrule(lr){2-9}
        & LoRA   & Multiplication & \cellcolor{lightblue}\textbf{71.45\textsubscript{$\pm$0.74}} & \cellcolor{lightblue}\textbf{63.64\textsubscript{$\pm$0.60}} & \cellcolor{lightblue}\textbf{70.83\textsubscript{$\pm$0.52}} & \cellcolor{lightgray}\underline{88.15\textsubscript{$\pm$0.73}} & \cellcolor{lightgray}\underline{74.94}\textsubscript{$\pm$0.83} & \cellcolor{lightblue}\textbf{73.80\textsubscript{$\pm$0.53}} \\
        & LoRA   & Addition     & 62.82\textsubscript{$\pm$0.62} & 50.78\textsubscript{$\pm$1.13} & 61.88\textsubscript{$\pm$0.41} & 60.25\textsubscript{$\pm$0.43} & 60.16\textsubscript{$\pm$0.94} & 59.18\textsubscript{$\pm$1.13} \\
        & LoRA   & Subtraction        & 63.63\textsubscript{$\pm$0.95} & 56.74\textsubscript{$\pm$0.72} & 66.05\textsubscript{$\pm$0.45} & 59.25\textsubscript{$\pm$0.95} & 63.35\textsubscript{$\pm$0.95} & 61.80\textsubscript{$\pm$0.57} \\
        & LoRA   & \textbf{CBI} & \cellcolor{lightgray}\underline{71.05\textsubscript{$\pm$0.36}} & \cellcolor{lightgray}\underline{58.78}\textsubscript{$\pm$0.98} & \cellcolor{lightgray}\underline{69.44\textsubscript{$\pm$1.33}} & \cellcolor{lightblue}\textbf{89.35\textsubscript{$\pm$0.56}} & \cellcolor{lightblue}\textbf{77.88\textsubscript{$\pm$0.62}} & \cellcolor{lightgray}\underline{73.30\textsubscript{$\pm$1.08}} \\
        \bottomrule
    \end{tabular}
    \caption{Temporal relation extraction under two PEFT settings (Adapter / LoRA), comparing different interaction methods that help routing features through the low-rank bottlenecks and their backbone LLMs. 
    The best score in each setting (with the same backbone and PEFT method) is shown in \cellcolor{lightblue}\textbf{bold blue}, the second-best is shown in \cellcolor{lightgray}\underline{gray underlined}.}
    \label{tab:temporal_routing_results_combined_full}
\end{table*}

\subsubsection{Data Efficiency and Robustness of CBI}
To understand how different interaction methods behave when training data is limited, we train each method using varying portions of the dataset (20\%–100\%). 
This allows us to observe how well each mechanism maintains its performance as the amount of available data decreases and how sharply its effectiveness drops under data-scarce conditions.
Figure~\ref{fig:adapter_lora_curve} shows the micro-F1 scores of Text Concatenation, Multiplication, and CBI under both the Adapter and LoRA configurations.\footnote{We omit Embedding Fusion due to its consistently inferior performance.}

In the Adapter setting, Text Concatenation remains highly sensitive to data scarcity, increasing from 42.12\% to 60.98\% as data grows. Multiplication is more stable (57.14\% to 60.73\%). CBI achieves the strongest performance, starting at 62.60\% with 20\% of the data and reaching 72.89\% with full supervision, showing substantially better generalization under limited annotation.

\textcolor{black}{The LoRA setting exhibits broadly similar patterns, although the differences between methods become smaller in scenarios with extremely limited annotation.
}
With complete data, CBI again achieves the best micro-F1 (77.88\%). 
At moderate data levels (e.g., 60\%), CBI reaches 67.56\%, outperforming Multiplication (65.35\%) and Text Concatenation (60.78\%). 
Under extremely limited annotation conditions, performance differences narrow, likely because LoRA introduces more trainable parameters than Adapter-based tuning, partially mitigating underfitting when supervision is sparse.

Across both tuning strategies, Text Concatenation shows the steepest degradation as data decreases, while CBI degrades much more gracefully. This validates our hypothesis that limited temporal-relation annotations reduce the ability of the model to align event semantics, and that feature-level interaction is crucial in low-resource settings.

Overall, CBI is both data-efficient and robust: it 
maintains strong performance in the PEFT setting when supervision is scarce. Its convolutional expansion and gated nonlinearity strengthen local event interactions while suppressing noise, resulting in resilient and interpretable temporal representations.

\subsubsection{Noise Filtering and Cue Localization in CBI}

We revisit the example introduced in Section~\ref{sec:intro} to analyze why CBI outperforms the baselines at the bottleneck representation level.
We visualize the event representations of each model along with their interpretable interaction maps in Figure~\ref{fig:heatmap_all}.
(a) \textbf{Multiplication interaction} (Figure~\ref{fig:multi_heatmap}):  
The interaction map assigns high importance to nearly all tokens in sentence 2, but fails to isolate the key cues. This observation aligns with our earlier analysis: while multiplication captures co-activation patterns, it also amplifies noise under limited supervision, making it difficult to highlight decisive evidence.
(b) \textbf{Additive interaction} (Figure~\ref{fig:add_heatmap}):  
The additive variant highlights parts of sentence 2, including ``disputes'' and ``has stalled'', yet still fails to precisely identify the most decisive tokens. This supports our claim that additive interaction tends to smooth distinctions, distributing attention over broader regions rather than producing sharp, focused signals.
(c) \textbf{Subtraction interaction} (Figure~\ref{fig:minus_heatmap}):  
This variant fails entirely, placing most of its attention on tokens in sentence 1. 
Without an explicit directional comparison between the two events, the model struggles to determine which event representation should be emphasized.
(d, e) \textbf{CBI interaction}:  
Before applying the two separate depthwise convolutions on the respective event representations (Figure~\ref{fig:dci_before_heatmap}), the model identifies ``bedeviled'' and ``disputes'' in sentence 1 as important, while assigning high importance to most tokens in sentence 2, similar to multiplication but still is differently focused.  
After applying the convolutions (Figure~\ref{fig:dci_after_heatmap}), irrelevant information is effectively filtered out. Attention to ``bedeviled'', ``has stalled'', ``disputes'', and ``trying'' is strongly amplified, while attention to other tokens is significantly suppressed.}

\textcolor{black}{These results show that depthwise convolution of CBI effectively filters noise and strengthens relevant cues, enabling more accurate temporal relation classification.
The theoretical analysis is found in Appendix~\ref{appendix:theory}. 
}

\begin{figure}[!t] 
\centering
\subfloat[\scriptsize{Event Interaction Map (Multiplication)}]{
  \includegraphics[width=\linewidth]{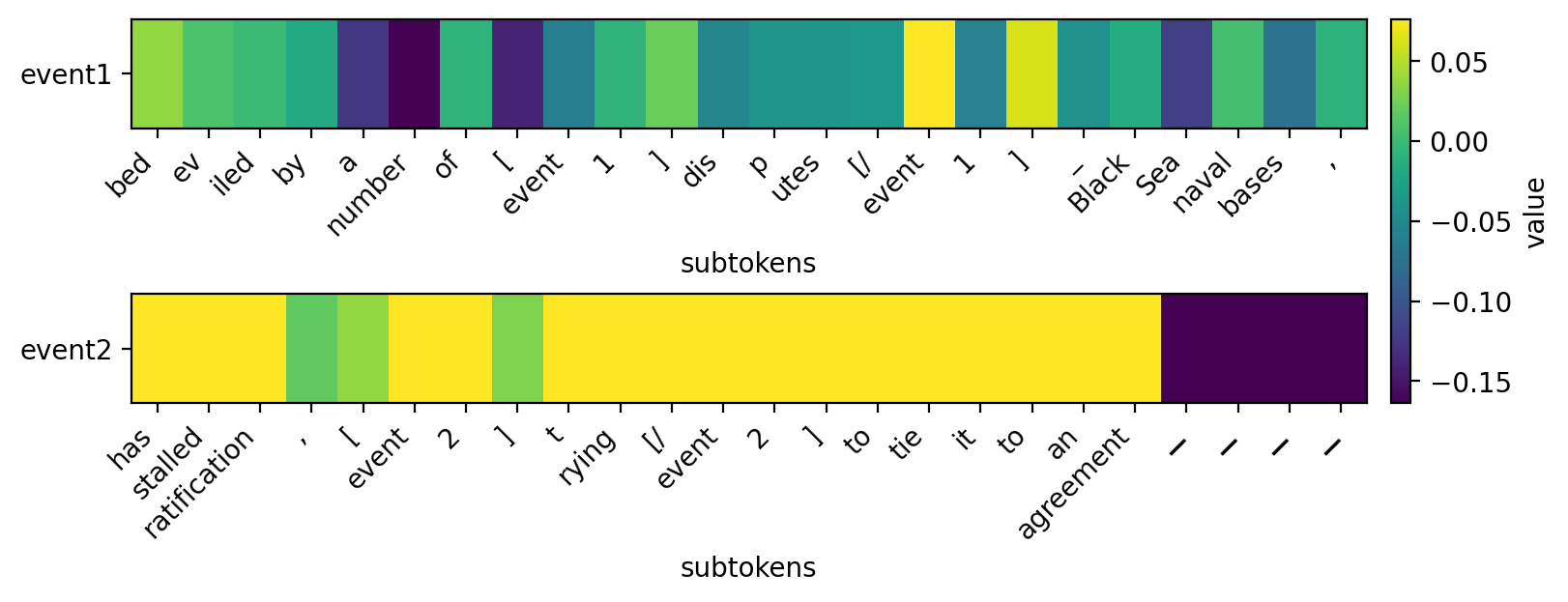}
  \label{fig:multi_heatmap}
}
\hfill
\subfloat[\scriptsize{Event Interaction Map (Addition)}]{
  \includegraphics[width=\linewidth]{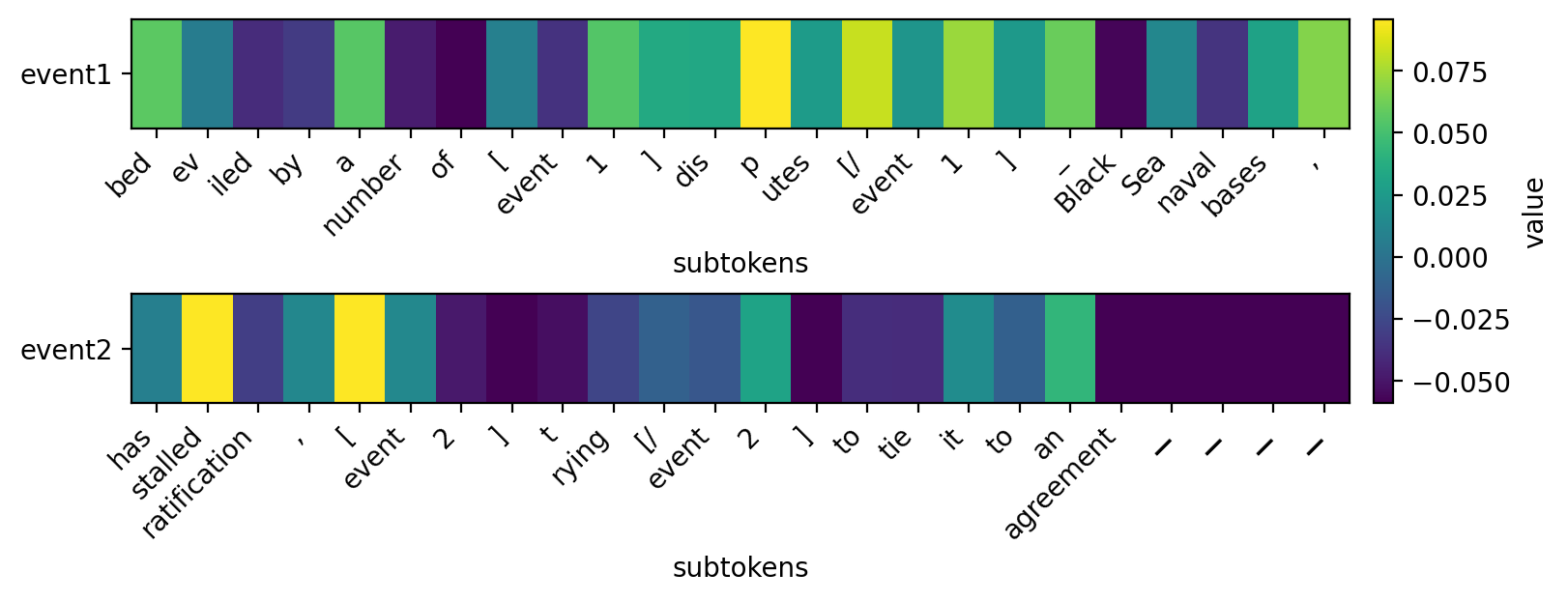}
  \label{fig:add_heatmap}
}

\medskip

\subfloat[\scriptsize{Event Interaction Map (Subtraction)}]{
  \includegraphics[width=\linewidth]{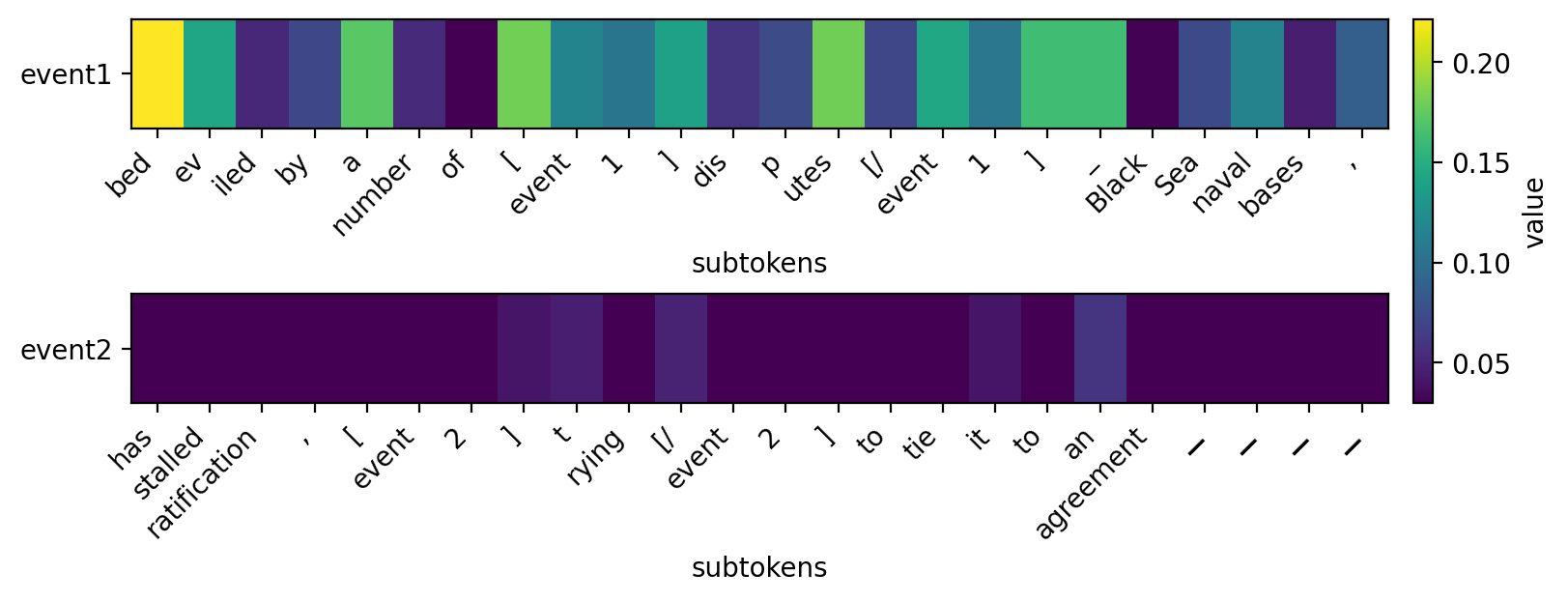}
  \label{fig:minus_heatmap}
}
\hfill
\subfloat[\scriptsize{Event Interaction Map (CBI before separate convolution)}]{
  \includegraphics[width=\linewidth]{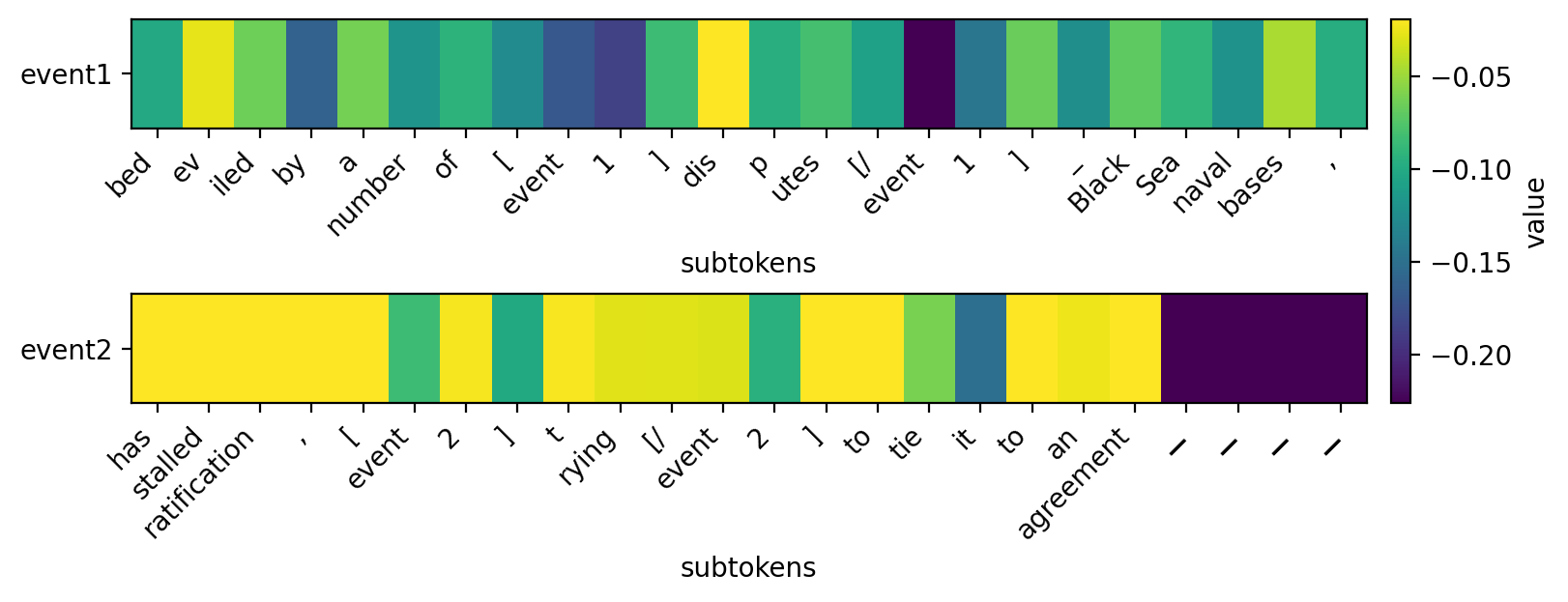}
  \label{fig:dci_before_heatmap}
}

\medskip

\subfloat[\scriptsize{Event Interaction Map (CBI after separate convolution)}]{
  \includegraphics[width=\linewidth]{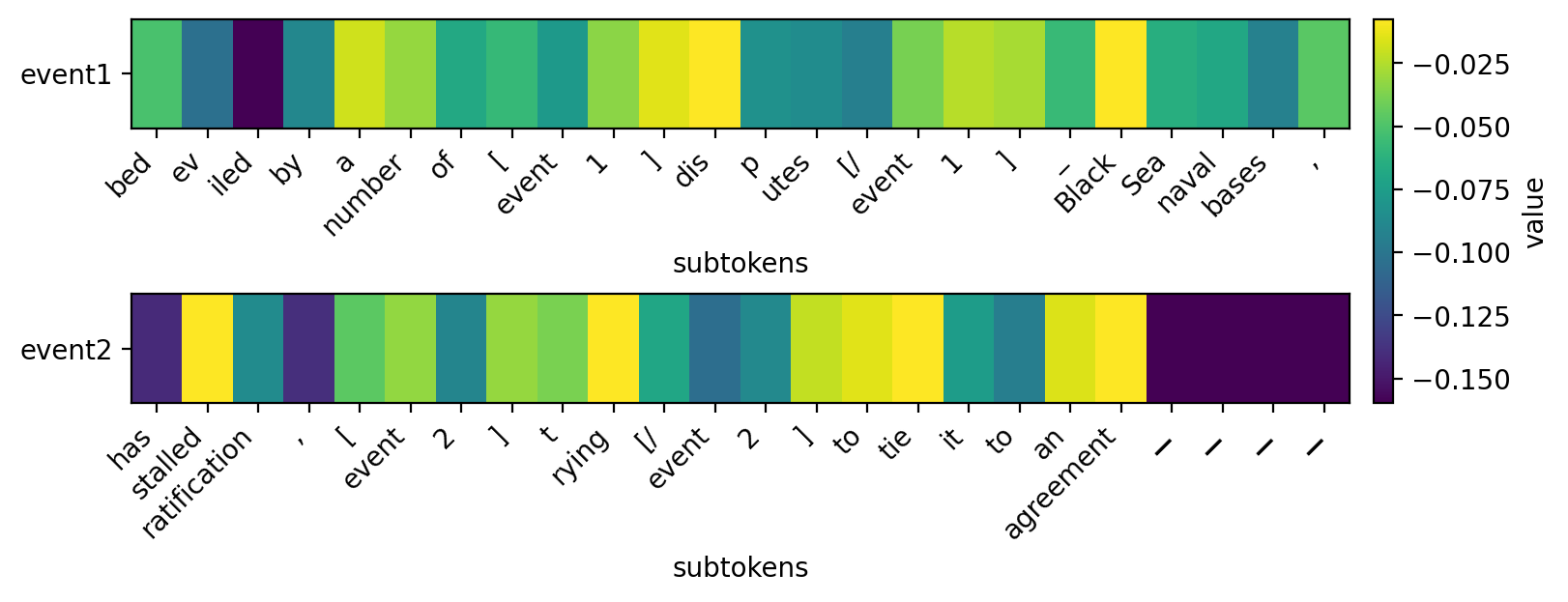}
  \label{fig:dci_after_heatmap}
}
\caption{Event interaction map for each event interaction method.}
\label{fig:heatmap_all}
\end{figure}

\begin{table*}[t]
\centering
\tiny
\begin{tabular}{llccccccc}
\toprule
\textbf{Computation Metrics} & \textbf{PEFT} & \textbf{Multiplication} & \textbf{Addition} & \textbf{CBI} & \textbf{Subtraction} & \textbf{Text Concatenation} & \textbf{Embedding Fusion} & \textbf{Full Fine-tuning} \\
\midrule
\multirow{2}{*}{\textbf{Memory Cost}} & LoRA    & 29269.37 & 29269.37 & 29271.05 & 29269.37 & 29271.40 & 29269.37 & \multirow{2}{*}{148792.17} \\
                             & Adapter & 29030.91 & 29030.91 & 29031.75 & 29030.91 & 29030.91 & 29030.91 & \\
\midrule
\multirow{2}{*}{\textbf{Time}}        & LoRA    & 223.84 & 216.71 & 221.12 & 217.16 & 375.81 & 214.07 & \multirow{2}{*}{3939.75} \\
                             & Adapter & 216.42 & 216.31 & 221.8  & 216.55 & 376.06 & 213.58  \\
\bottomrule
\end{tabular}
\caption{Computational analysis of LoRA and Adapter on Llama-3.1-8B, reporting GPU memory usage (MB) and runtime in terms of seconds (s) across six temporal relation extraction models.}
\label{tab:cost_time}
\end{table*}

\subsection{Sensitivity Analysis}

\begin{table}[t]
\centering
\tiny
\setlength{\tabcolsep}{2pt}
\begin{tabular}{c| c c c c c| c}
\toprule
\textbf{Kernel Size} & \textbf{MATRES} & \textbf{TB-Dense} & \textbf{TCR} & \textbf{TDD-Auto} & \textbf{TDD-Man} & \textbf{AVG} \\
\midrule
\multicolumn{7}{c}{\textbf{Adapter}} \\
\midrule
7 
& \cellcolor{lightgray}{\underline{67.90}}
& \cellcolor{lightgray}{\underline{56.11}}
& \cellcolor{lightgray}{\underline{63.89}}
& \cellcolor{lightgray}{\underline{88.80}}
& \cellcolor{lightgray}{\underline{71.43}}
& \cellcolor{lightblue}{\textbf{69.63}} \\
5 
& 64.84 
& 55.80
& 63.43
& \cellcolor{lightblue}{\textbf{89.20}} 
& \cellcolor{lightblue}{\textbf{72.98}} 
& 69.25 \\
3 
& \cellcolor{lightblue}{\textbf{66.13}} & \cellcolor{lightblue}{\textbf{58.31}} & \cellcolor{lightblue}{\textbf{65.58}} & 86.35 & 70.45 & \cellcolor{lightgray}{\underline{69.36}} \\
1 
& 62.34 & 54.86 & 62.96 & 78.75 & 61.31 & 64.04 \\
\midrule
\multicolumn{7}{c}{\textbf{LoRA}} \\
\midrule
7 
& 69.76
& 57.05
& 67.44
& \cellcolor{lightgray}{\underline{88.80}}
& \cellcolor{lightgray}{\underline{77.63}}
& 72.14 \\
5 
& \cellcolor{lightblue}{\textbf{71.05}} 
& 58.78
& \cellcolor{lightblue}{\textbf{69.44}} 
& \cellcolor{lightblue}{\textbf{89.35}} 
& \cellcolor{lightblue}{\textbf{77.88}} 
& \cellcolor{lightblue}{\textbf{73.30}} \\
3 
& \cellcolor{lightgray}{\underline{69.92}}
& \cellcolor{lightgray}{\underline{59.72}}
& \cellcolor{lightgray}{\underline{68.52}}
& 87.55 
& 76.57 
& \cellcolor{lightgray}{\underline{72.46}} \\
1 
& 69.52 & \cellcolor{lightblue}{\textbf{59.40}} & 66.36 & 84.65 & 72.16 & 70.42 \\
\bottomrule
\end{tabular}
\caption{
Influence of the convolutional kernel size for CBI under Adapter and LoRA settings. Kernel size 5 provides the most stable and competitive performance across all datasets. Best results are in \cellcolor{lightblue}{\textbf{bold blue}}, second-best in \cellcolor{lightgray}{\underline{gray underlined}}.}
\label{tab:ablation_kernel}
\end{table}

To assess the role of convolutional refinement within CBI, we conduct a sensitivity analysis on the TDD-man dataset with Llama 3.1 8B. 
We analyze how performance varies with different convolutional kernel sizes across five temporal relation extraction benchmarks.

We vary the convolutional kernel size from 1 to 7 to assess its influence on temporal relation modeling. As shown in Table \ref{tab:ablation_kernel}, extremely small kernels (e.g., size 1) significantly underperform in both Adapter and LoRA settings, indicating that minimal receptive fields provide insufficient temporal context. 
Kernel size 5 performs reliably across all five datasets, whereas kernel size 7 exhibits greater variance in average micro-F1. This stability is essential for cross-dataset generalization, and therefore we adopt kernel size 5 as the default configuration of CBI.

\subsection{Computational Analysis}

We compare several interaction-based designs in PEFT settings to assess their efficiency and effectiveness (Table~\ref{tab:cost_time}). Multiplication and CBI substantially reduce training time by more than 40\% compared to Text Concatenation while keeping GPU memory usage nearly unchanged. Text Concatenation processes concatenated event pairs through a low-rank bottleneck without explicit interactions, which leads to much longer training (over 370 seconds per epoch) and weaker performance. Other lightweight variants (Embedding Fusion, Subtraction, and Addition) achieve similar computational cost but capture fewer inter-event dependencies, resulting in slightly inferior accuracy. Building on Multiplication, CBI adds only two depthwise convolutional layers with small kernels, increasing parameters by less than 0.1\% and introducing negligible overhead. As shown in Table~\ref{tab:cost_time}, CBI matches the efficiency of Multiplication while consistently achieving the best performance, demonstrating that explicitly modeling event interactions within the low-rank bottleneck is both computationally efficient and crucial for effective temporal reasoning.

\section{Conclusion}
In this work, we analyzed why PEFT performs poorly on temporal relation extraction despite the strong general ability of large language models. We found that low-rank bottlenecks block cross-event interactions, preventing key cues such as discourse markers from being properly encoded.
To overcome this limitation, we proposed Convolutional Bottleneck Interaction (CBI), a simple architecture that injects event–event interaction directly within  the bottleneck. CBI first applies a lightweight depthwise convolution to enlarge the receptive field of each event and to reduce local noise, and then performs element-wise multiplication to model channel-wise co-activation between events. This design addresses the two main weaknesses of multiplicative interaction, token misalignment and noise amplification, while keeping the computational cost minimal.
Experiments on five datasets and seven LLM backbones with adapters and LoRA show that CBI delivers large and consistent gains, up to +31.7 micro-F1, with negligible overhead.
These results demonstrate that explicitly modeling event interaction within the low-rank space is both effective and necessary for temporal relation extraction in PEFT. 
More broadly, our findings suggest that relational tasks benefit from interaction-aware bottleneck designs, offering a promising direction for future parameter-efficient architectures.

\textcolor{black}{Beyond temporal relation extraction, the idea of modeling interactions within low-rank bottlenecks may benefit a wider range of relational reasoning tasks. Many NLP problems require capturing dependencies between events, entities, or discourse units, where effective interaction modeling is essential. Future work could explore extending interaction-aware bottleneck designs to tasks such as causal relation extraction, discourse relation classification, and document-level reasoning.}

\section*{Acknowledgements}
This research were funded by the AIDAVA (EU HORIZON-HLTH-2021-TOOL-06-03), the CHIST-ERA iTRUST project (CHIST-ERA Call 2021, FWO grant G0L0822N) and Flanders Artificial Intelligence Research program (FAIR).

The computations described in this research were performed using the The Flemish Supercomputer Center Tier-1 HPC service~(https://www.vscentrum.be/) and the high-performance computing resources provided by CINECA~(https://www.cineca.it/), Italy.

Finally, we would like to thank Minye Wu, as well as our action editor and anonymous reviewers, for their time, effort, and valuable feedback.

\bibliography{tacl2021}
\bibliographystyle{acl_natbib}

\section{Theoretical Analysis}
\label{appendix:theory}

\subsection{Motivation}

PEFT methods such as LoRA and Adapter
restrict learnable updates to a low-rank subspace of pretrained models.
Although efficient, this {ranking bottleneck} limits the ability of the model
to learn high-rank cross-event interactions required in temporal relation extraction.
In particular, when two event mentions are represented by latent sequences,
a full bilinear interaction between them is generally high-rank,
whereas PEFT updates cannot reliably parameterize such transformations.

To alleviate this deficiency, we propose Convolutional Bottleneck Interaction (CBI),
which introduces two inductive biases on the {feature map}:
(a) {Depthwise convolution} for local temporal alignment and noise suppression;
(b) {Element-wise multiplication} for explicit second-order event interaction.
We formalize why these operations restore relational expressivity
without increasing the PEFT rank.

\subsection{Preliminaries and Notation}

Let $\mathbf{H}_1,\mathbf{H}_2\in\mathbb{R}^{L\times d}$ be the contextual token
representations of two event spans produced by the backbone transformer,
where $L$ is the span length and $d$ the hidden size.
PEFT introduces a low-rank bottleneck via a projection matrix
$\mathbf{W}_{down}\in\mathbb{R}^{d\times r}$, $r\ll d$, yielding
\begin{equation}
    \mathbf{E}_i = \mathbf{H}_i\mathbf{W}_{down} \in \mathbb{R}^{L\times r},\quad i\in\{1,2\}.
\end{equation}
We denote by $\mathbf{e}_i(l)\in\mathbb{R}^r$ the $l$-th token vector of $\mathbf{E}_i$:
\begin{equation}
    \mathbf{e}_i(l) = \mathbf{E}_i[l,:].
\end{equation}

\paragraph{Ideal bilinear interaction:}
A general second-order cross-event interaction can be written as a bilinear form
\begin{multline}
    s^\star(\mathbf{E}_1,\mathbf{E}_2)
    = \left\langle \mathbf{W},\, \mathbf{E}_1^{\top}\mathbf{E}_2 \right\rangle \\
    = \sum_{c_1=1}^r \sum_{c_2=1}^r
      \mathbf{W}_{c_1,c_2}\,
      \left\langle \mathbf{E}_1[:,c_1], \mathbf{E}_2[:,c_2]\right\rangle,
\end{multline}
where $\mathbf{W}\in\mathbb{R}^{r\times r}$
is a task-specific bilinear weight matrix and
$\langle\cdot,\cdot\rangle$ denotes the standard inner product.
When $r$ is small due to PEFT, it is difficult to directly learn full $\mathbf{W}$
(or equivalently complete cross-channel mixing) is difficult.

\paragraph{CBI operations:}
CBI first applies a depthwise convolution along the token axis.
Let $\mathbf{T}\in\mathbb{R}^{L\times L}$ be the Toeplitz matrix
associated with a 1D depthwise kernel of width $k$.
For each event:
\begin{multline}\label{eq:conv_def_1}
    \tilde{\mathbf{E}}_i = \mathbf{T}\mathbf{E}_i \in\mathbb{R}^{L\times r},
    \quad\text{equivalently}\quad \\
    \tilde{\mathbf{e}}_i(l)=\sum_{j} \theta_j\,\mathbf{e}_i(l+j).
\end{multline}
Then CBI constructs the interaction feature field
\begin{multline}\label{eq:mul_def}
    \mathbf{F} = \tilde{\mathbf{E}}_1 \odot \tilde{\mathbf{E}}_2\in\mathbb{R}^{L\times r},\\
    \qquad
    \mathbf{F}[l,c] = \tilde{\mathbf{E}}_1[l,c]\tilde{\mathbf{E}}_2[l,c].
\end{multline}
Finally, any downstream output (e.g., pooling or attention-based aggregation)
maps $\mathbf{F}$ to a scalar score:
\begin{equation}\label{eq:readout_def}
    s_{\mathrm{CBI}}(\mathbf{E}_1,\mathbf{E}_2)
    = \langle \mathbf{p},\,\mathrm{vec}(\mathbf{F})\rangle,
\end{equation}
where $\mathrm{vec}(\mathbf{F})\in\mathbb{R}^{Lr}$ stacks all tokens and channels
into a vector, and $\mathbf{p}\in\mathbb{R}^{Lr}$ is a vector of learned linear weights.
We now show that $\mathbf{F}$ provides an explicit bilinear interaction
compatible with such outputs.

\subsection{Element-wise Multiplication as Explicit Diagonal Bilinear Interaction}

We first characterize the bilinear structure induced by $\mathbf{F}$.

\begin{lemma}[Diagonal bilinear identity]\label{lem:diag_identity}
For each channel $c\in\{1,\dots,r\}$,
\begin{equation}\label{eq:diag_identity}
    \sum_{l=1}^L \mathbf{F}[l,c]
    = \left(\tilde{\mathbf{E}}_1^{\top}\tilde{\mathbf{E}}_2\right)_{c,c}.
\end{equation}
\end{lemma}

\noindent\textbf{Proof.}
Let channel $c$ be arbitrary.
By definition of matrix multiplication,
\begin{equation}
\left(\tilde{\mathbf{E}}_1^{\top}\tilde{\mathbf{E}}_2\right)_{c,c}
= \sum_{l=1}^L \tilde{\mathbf{E}}_1[l,c]\tilde{\mathbf{E}}_2[l,c].
\end{equation}
By the definition of $\mathbf{F}$ in \eqref{eq:mul_def},
$\tilde{\mathbf{E}}_1[l,c]\tilde{\mathbf{E}}_2[l,c]=\mathbf{F}[l,c]$.
This yields \eqref{eq:diag_identity}.

\paragraph{Consequence.}
Lemma~\ref{lem:diag_identity} shows that $\mathbf{F}$
is an explicit token-wise representation of diagonal bilinear interactions.
Thus, any linear output that aggregates tokens within each channel (a special case of
\eqref{eq:readout_def}) can directly realize a diagonal bilinear score
\begin{multline}\label{eq:diag_score}
    s_{\mathrm{diag}}(\mathbf{E}_1,\mathbf{E}_2)
    = \sum_{c=1}^r \mathbf{w}_c\,
      \left(\tilde{\mathbf{E}}_1^{\top}\tilde{\mathbf{E}}_2\right)_{c,c}\\
    = \sum_{c=1}^r \mathbf{w}_c \sum_{l=1}^L \mathbf{F}[l,c],
\end{multline}
for some $\mathbf{w}\in\mathbb{R}^{r}$.

\begin{lemma}[Linear recoverability]\label{lem:recover}
For any $\mathbf{w}\in\mathbb{R}^r$,
there exists $\mathbf{p}_{\mathbf{w}}\in\mathbb{R}^{Lr}$ such that
\begin{equation}
    s_{\mathrm{diag}}(\mathbf{E}_1,\mathbf{E}_2)
    = \langle \mathbf{p}_{\mathbf{w}},\,\mathrm{vec}(\mathbf{F})\rangle,
\end{equation}
\end{lemma}
\noindent where $\mathrm{vec}$ denotes column-wise vectorization.

\noindent\textbf{Proof.}
Construct $\mathbf{p}_{\mathbf{w}}$ by assigning weight $w_c$
to all $L$ entries corresponding to channel $c$ in $\mathrm{vec}(\mathbf{F})$,
and zero elsewhere.
Then
\[
\langle \mathbf{p}_{\mathbf{w}},\mathrm{vec}(\mathbf{F})\rangle
= \sum_{c=1}^r w_c \sum_{l=1}^L \mathbf{F}[l,c]
= s_{\mathrm{diag}}(\mathbf{E}_1,\mathbf{E}_2),
\]
using \eqref{eq:diag_score}.

\paragraph{Interpretation.}
Element-wise multiplication converts two first-order representations
into an explicit second-order interaction field.
Any later linear aggregation can pick up these diagonal bilinear signals,
whereas a PEFT-constrained backbone would otherwise need to implicitly
provide them through low-rank updates.

\subsection{Depthwise Convolution as Soft Alignment and Denoising}

We next show that depthwise convolution preserves diagonal interactions
up to a controlled error and improves their signal-to-noise ratio.

\paragraph{Signal-to-noise decomposition.}
Write each low-rank event sequence as
\begin{equation}
    \mathbf{E}_i = \mathbf{S}_i + \boldsymbol{\varepsilon}_i,
\end{equation}
where $\mathbf{S}_i$ is the latent semantic signal relevant to temporal relations,
and $\boldsymbol{\varepsilon}_i$ represents noise.

\begin{lemma}[Stability of diagonal interaction]\label{lem:stability}
\label{sec:multi_error}
Assume $\|\mathbf{T}^{\top}\mathbf{T}-\mathbf{I}\|_2\le\epsilon$.
Then for any $\mathbf{S}_1,\mathbf{S}_2\in\mathbb{R}^{L\times r}$,
\begin{multline}
\big\|\mathrm{diag}(\mathbf{S}_1^{\top}\mathbf{T}^{\top}\mathbf{T}\mathbf{S}_2)
      -\mathrm{diag}(\mathbf{S}_1^{\top}\mathbf{S}_2)\big\|_2 \\
\le \epsilon\,\|\mathbf{S}_1\|_F\,\|\mathbf{S}_2\|_F .
\end{multline}
\end{lemma}

\noindent\textbf{Proof.}
\[
\begin{aligned}
&\mathrm{diag}(\mathbf{S}_1^{\top}\mathbf{T}^{\top}\mathbf{T}\mathbf{S}_2)
 -\mathrm{diag}(\mathbf{S}_1^{\top}\mathbf{S}_2) \\
&=\mathrm{diag}\big(\mathbf{S}_1^{\top}(\mathbf{T}^{\top}\mathbf{T}-\mathbf{I})\mathbf{S}_2\big).
\end{aligned}
\]
Taking norms and using $\|\mathrm{diag}(\mathbf{A})\|_2\le\|\mathbf{A}\|_F$,
\[
\begin{aligned}
\|\cdots\|_2
&\le \|\mathbf{S}_1^{\top}(\mathbf{T}^{\top}\mathbf{T}-\mathbf{I})\mathbf{S}_2\|_F \\
&\le \|\mathbf{T}^{\top}\mathbf{T}-\mathbf{I}\|_2\,\|\mathbf{S}_1\|_F\,\|\mathbf{S}_2\|_F\\
&\le \epsilon\,\|\mathbf{S}_1\|_F\,\|\mathbf{S}_2\|_F.
\end{aligned}
\]

\paragraph{Meaning.}
Lemma~\ref{lem:stability} formalizes soft alignment:
a low-pass depthwise kernel makes $\mathbf{T}^{\top}\mathbf{T}$ close to $\mathbf{I}$
on smooth signals, so diagonal bilinear interactions are approximately preserved.

\begin{lemma}[Noise contraction and SNR improvement]\label{lem:snr}
\label{sec:cbi_solve_error}
Assume $\|\mathbf{T}\boldsymbol{\varepsilon}_i\|_2^2
      \le \alpha\|\boldsymbol{\varepsilon}_i\|_2^2$ for some $0<\alpha<1$.
Then the expected diagonal noise interaction energy satisfies
\begin{multline}
\mathbb{E}\big\|\mathrm{diag}((\mathbf{T}\boldsymbol{\varepsilon}_1)^{\top}
(\mathbf{T}\boldsymbol{\varepsilon}_2))\big\|_2^2
\le \\ 
\alpha^2\,
\mathbb{E}\big\|\mathrm{diag}(\boldsymbol{\varepsilon}_1^{\top}
\boldsymbol{\varepsilon}_2)\big\|_2^2 .
\end{multline}
Consequently, the signal-to-noise ratio
of diagonal interactions after convolution is improved by at least $1/\alpha^2$.
\end{lemma}

\noindent\textbf{Proof.}
Under the assumption,
the energy of each noise sequence 
is reduced by a factor $\alpha$ after applying $\mathbf{T}$.
Since diagonal interaction is a sum of products of corresponding noise entries,
its second-moment energy scales multiplicatively, giving a factor $\alpha^2$.

\subsection{Effective Second-order Expressivity under PEFT Constraints}

PEFT constrains parameter updates to rank-$r$ forms, e.g.,
\begin{equation}
\Delta\mathbf{W}=\mathbf{W}_{\text{down}}\mathbf{W}_{\text{up}},
\end{equation}
where $\mathbf{W}_{\text{down}}\in\mathbb{R}^{d\times r}$ is the down-projection and  
$\mathbf{W}_{\text{up}}\in\mathbb{R}^{r\times d}$ is the up-projection, with $r\ll d$. 
This limits the ability of the backbone to create high-rank bilinear interactions
between two events through parameter adaptation alone.

CBI instead produces second-order interactions in the feature space:
\begin{multline}
\boldsymbol{\Phi}_{\mathrm{CBI}}(\mathbf{H}_1,\mathbf{H}_2)
=\\
(\mathbf{T}\mathbf{H}_1\mathbf{W}_{\text{down}})
\odot
(\mathbf{T}\mathbf{H}_2\mathbf{W}_{\text{down}})
=
\mathbf{F}.
\end{multline}
By Lemma~\ref{lem:diag_identity}, $\mathbf{F}$ explicitly encodes
diagonal bilinear interactions,
and by Lemma~\ref{lem:recover}, any linear output can exploit them directly.
Lemmas~\ref{lem:stability}--\ref{lem:snr} further show that depthwise convolution
stabilizes these interactions under token misalignment and improves their signal-to-noise ratio.
Therefore, CBI restores second-order relational expressivity
while keeping the PEFT rank unchanged and adding negligible parameters.

\subsection{Generalizing CBI to Graph Neural Network (GNN)–Based Temporal Relation Extraction Models}

\begin{table}[t]
\centering
\tiny
\renewcommand{\arraystretch}{1}
\setlength{\tabcolsep}{1pt}
\begin{tabular}{l| c c c c c| c}
\toprule
\textbf{Methods} & \textbf{MATRES} & \textbf{TB-Dense} & \textbf{TCR} & \textbf{TDD-Auto} & \textbf{TDD-Man} & \textbf{AVG} \\
\midrule
\multicolumn{7}{c}{\textbf{Adapter}} \\
\midrule
Text Concatenation 
& \cellcolor{lightgray}{\underline{60.25}} 
& \cellcolor{lightgray}{\underline{51.25}} 
& \cellcolor{lightgray}{\underline{61.49}} 
& \cellcolor{lightgray}{\underline{65.77}} 
& \cellcolor{lightgray}{\underline{64.41}} 
& \cellcolor{lightgray}{\underline{60.63}} \\
Embedding Fusion 
& 57.70 & 43.13 & 59.76 & 48.31 & 37.20 & 49.22 \\
\textbf{CBI} 
& \cellcolor{lightblue}{\textbf{66.61}} 
& \cellcolor{lightblue}{\textbf{56.23}} 
& \cellcolor{lightblue}{\textbf{65.15}} 
& \cellcolor{lightblue}{\textbf{90.02}} 
& \cellcolor{lightblue}{\textbf{75.99}} 
& \cellcolor{lightblue}{\textbf{70.80}} \\
\midrule
\multicolumn{7}{c}{\textbf{LoRA}} \\
\midrule
Text Concatenation 
& \cellcolor{lightgray}{\underline{66.24}} 
& \cellcolor{lightblue}{\textbf{59.41}} 
& \cellcolor{lightgray}{\underline{63.04}} 
& \cellcolor{lightgray}{\underline{72.02}} 
& \cellcolor{lightgray}{\underline{64.50}} 
& \cellcolor{lightgray}{\underline{65.04}} \\
Embedding Fusion 
& 56.20 & 41.97 & 59.76 & 36.74 & 36.73 & 46.28 \\
\textbf{CBI} 
& \cellcolor{lightblue}{\textbf{72.37}} 
& \cellcolor{lightgray}{\underline{58.93}} 
& \cellcolor{lightblue}{\textbf{68.36}} 
& \cellcolor{lightblue}{\textbf{90.09}} 
& \cellcolor{lightblue}{\textbf{78.79}} 
& \cellcolor{lightblue}{\textbf{73.71}} \\
\bottomrule
\end{tabular}
\caption{Performance comparison of Convolutional Bottleneck Interaction (CBI) with two baselines under Adapter and LoRA settings augmented with inference over a global document-level event graph applied on five temporal relation extraction datasets (micro-F1\%) using the \textbf{Llama 3.1 8B} backbone. Best results are in \cellcolor{lightblue}{\textbf{bold blue}}, second-best in \cellcolor{lightgray}{\underline{gray underlined}}.}
\label{tab:main_results_llama3_slim}
\end{table}

We also investigate whether our interaction methods generalize to GNN-based temporal relation extraction, focusing on GDLLM~\citep{zhao2025gdllm}, a recent state-of-the-art model.
GDLLM first uses an LLM to encode all event mentions in a document. It then builds a global event graph, where nodes are events, and edges reflect their discourse distance. A distance-aware GNN propagates information across this graph, helping the model reason about long-range temporal dependencies.

To test our approach in this setting, we replace the GDLLM event-pair encoder with three interaction strategies: Text Concatenation, Embedding Fusion, and our CBI module in both Adapter and LoRA configurations.
The results in Table~\ref{tab:main_results_llama3_slim} show three clear trends.
First, Text Concatenation performs reasonably well, but the improvements remain limited. Even with a strong GNN on top, simply merging two event representations does not provide enough direct interaction between events.
Second, Embedding Fusion performs the worst in all settings. This confirms that shallow element-wise operations are easily affected by the low-rank constraints of PEFT modules, causing important temporal cues to be lost before the GNN can use them.
Third, CBI achieves the best results in all datasets and all PEFT configurations. 
These improvements suggest that CBI provides cleaner and more informative pair representations, which the GNN can then integrate more effectively during global message passing.

Overall, even after adding a GNN module, our method remains the strongest interaction mechanism. 
This shows that better event interaction inside the PEFT bottleneck is crucial, and that CBI offers an effective and lightweight solution in both transformer-only and GNN-augmented settings.




\section{Dataset statistics}
\label{sec:dataset_stats}
We conduct experiments on five widely used datasets for temporal relation extraction, whose statistics are provided in Table~\ref{tab:dataset_stats}.

\begin{table}[ht]
\centering
\begin{tabular}{lccc}
\toprule
\textbf{Dataset} & \textbf{Train} & \textbf{Dev} & \textbf{Test} \\
\midrule
MATRES    & 3{,}715  & 1{,}240 & 1{,}238 \\
TB-Dense  & 1{,}910  & 638     & 636     \\
TCR       & 1{,}939  & 648     & 646     \\
TDD-Auto  & 22{,}981 & 7{,}660 & 7{,}661 \\
TDD-Man   & 3{,}673  & 1{,}225 & 1{,}224 \\
\bottomrule
\end{tabular}
\caption{Statistics of the five temporal relation extraction datasets in terms of number of temporal relation examples.}
\label{tab:dataset_stats}
\end{table}

\section{Experimental Analysis}
\subsection{Results of Varying Bottleneck Capacity}
\textcolor{black}{Table~\ref{tab:bottleneck_parallel} shows that increasing the bottleneck rank $r$ brings only small improvements. For example, when $r$ increases from 128 to 512, Text Concatenation improves by about 3 points in both Adapter and LoRA settings. Embedding Fusion shows similarly limited gains. However, even when $r$ is larger, these methods still perform much worse than CBI. In contrast, CBI achieves clearly better results without increasing the rank. It indicates that performance improvement of CBI does not come from a larger parameter capacity, but from a more effective event interaction mechanism.}
\begin{table}[t]
\centering
\tiny
\renewcommand{\arraystretch}{1}
\setlength{\tabcolsep}{3pt}
\begin{tabular}{l|ccc|ccc}
\toprule
\textbf{Methods} 
& \multicolumn{3}{c|}{\textbf{Adapter (rank $r$)}} 
& \multicolumn{3}{c}{\textbf{LoRA (rank $r$)}} \\
& $r$=128 & $r$=256 & $r$=512 & $r$=128 & $r$=256 & $r$=512 \\
\midrule
Text Concatenation 
& 60.06 & \cellcolor{lightgray}{\underline{61.87}} & \cellcolor{lightgray}{\underline{63.36}} 
& 64.94 & \cellcolor{lightgray}{\underline{65.92}} & \cellcolor{lightgray}{\underline{66.32}} \\

Embedding Fusion 
& 42.69 & 43.37 & 44.14 
& 41.61 & 42.78 & 43.94 \\

\textbf{CBI} 
& \cellcolor{lightblue}{\textbf{69.25}} & \cellcolor{lightblue}{\textbf{70.31}} & \cellcolor{lightblue}{\textbf{71.34}} 
& \cellcolor{lightblue}{\textbf{73.30}} & \cellcolor{lightblue}{\textbf{73.94}} & \cellcolor{lightblue}{\textbf{74.71}} \\
\bottomrule
\end{tabular}
\caption{\textcolor{black}{Performance comparison of CBI results with results baseline methods obtained with different bottleneck ranks $r$ under Adapter and LoRA settings (the averaged micro-F1 score across all five benchmarks). Best results are in \cellcolor{lightblue}{\textbf{bold blue}}, second-best in \cellcolor{lightgray}{\underline{gray underlined}}.}}
\label{tab:bottleneck_parallel}
\end{table}

\subsection{Error Analysis}

\begin{figure*}[ht]
\centering
\subfloat[\tiny{Full Fine-tuning.}]{\includegraphics[width=0.65\textwidth]{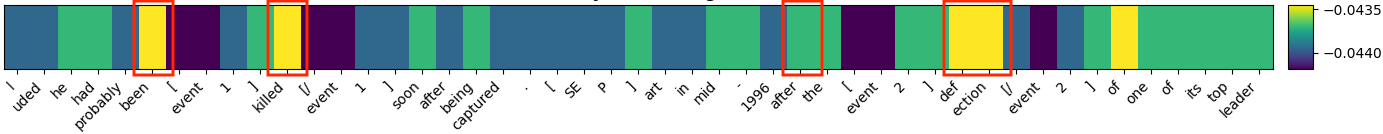}
\label{fig:full_ft_sample_11}
}
\hfil
\subfloat[\tiny{Text Concatenation with Adapter fine-tuning.}]{\includegraphics[width=0.65\textwidth]{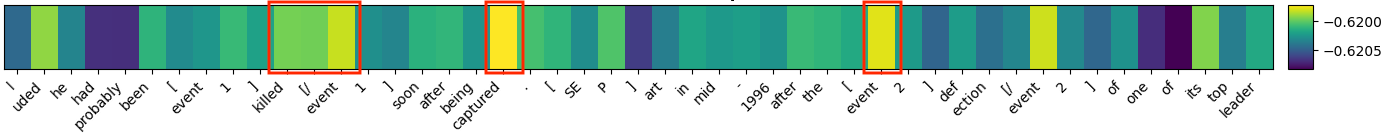}
\label{fig:text_concatenation_adapter_11}
}
\hfil
\subfloat[\tiny{Text Concatenation with LoRA fine-tuning.}]{\includegraphics[width=0.65\textwidth]{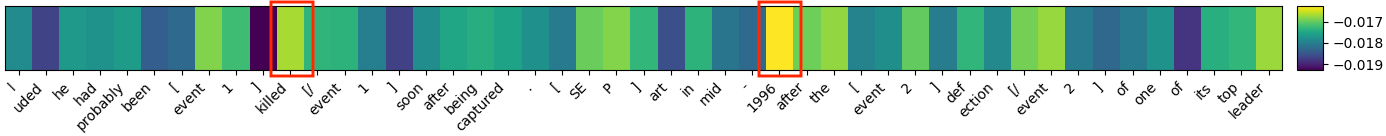}
\label{fig:text_concatenation_lora_11}
}
\hfil
\subfloat[\tiny{Embedding Fusion with Adapter fine-tuning.}]{\includegraphics[width=0.4\textwidth]{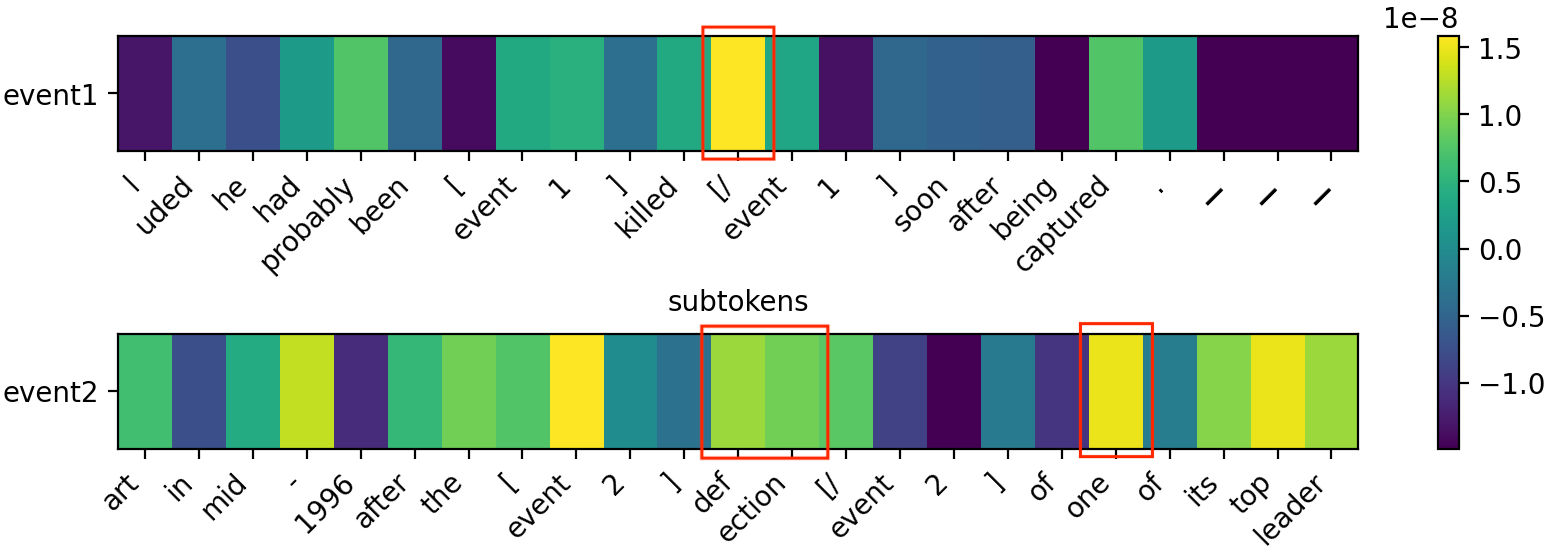}
\label{fig:embedding_adapter_11}
}
\hfil
\subfloat[\tiny{Embedding Fusion with Lora fine-tuning.}]{\includegraphics[width=0.4\textwidth]{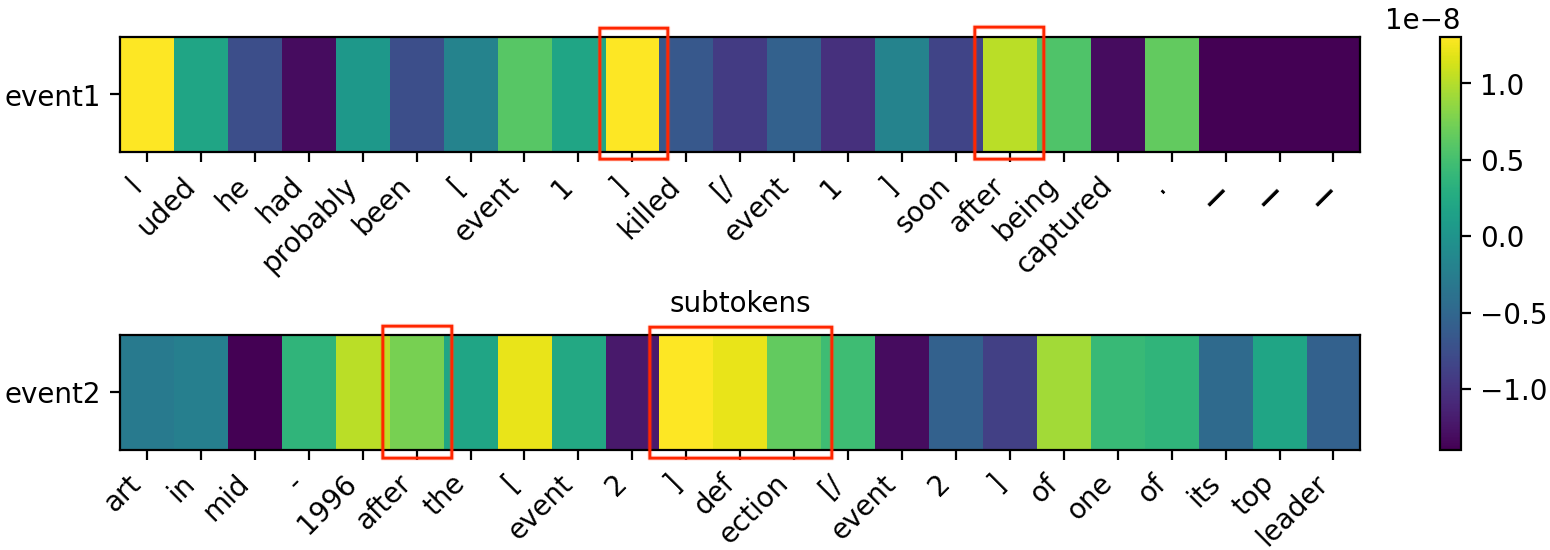}
\label{fig:embedding_lora_11}
}
\hfil
\subfloat[\tiny{CBI with Adapter fine-tuning.}]{\includegraphics[width=0.4\textwidth]{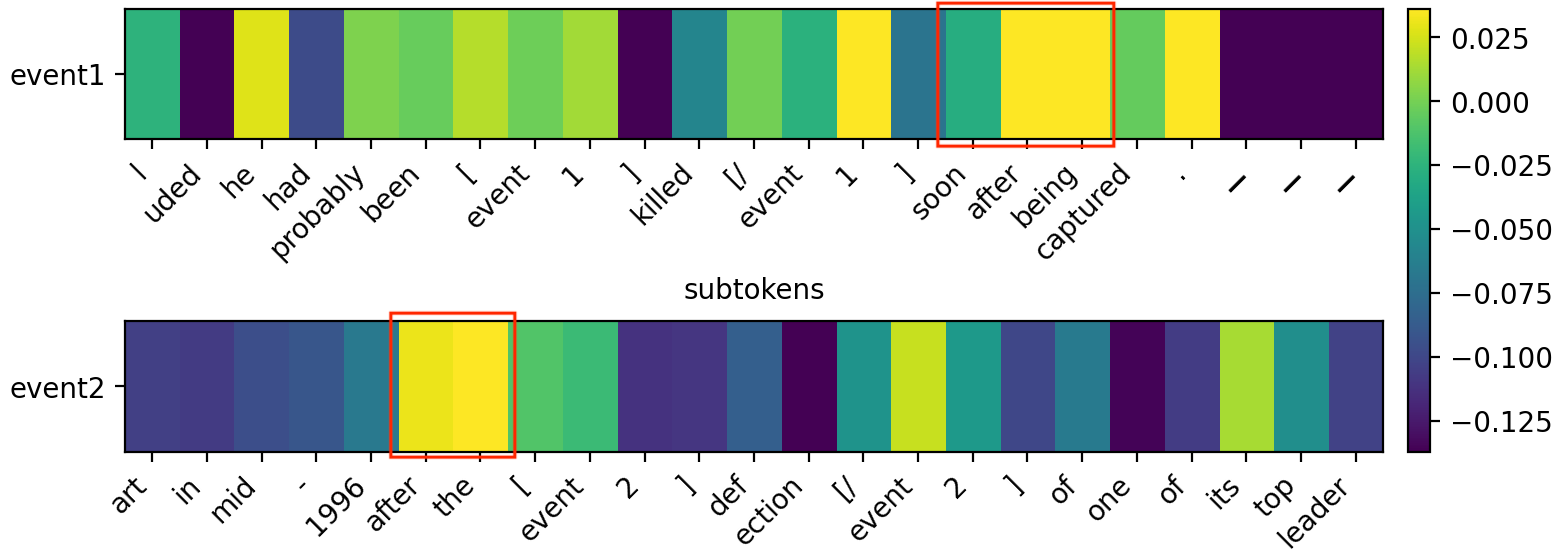}
\label{fig:cbi_adapter_11}
}
\hfil
\subfloat[\tiny{CBI with LoRA fine-tuning.}]{\includegraphics[width=0.4\textwidth]{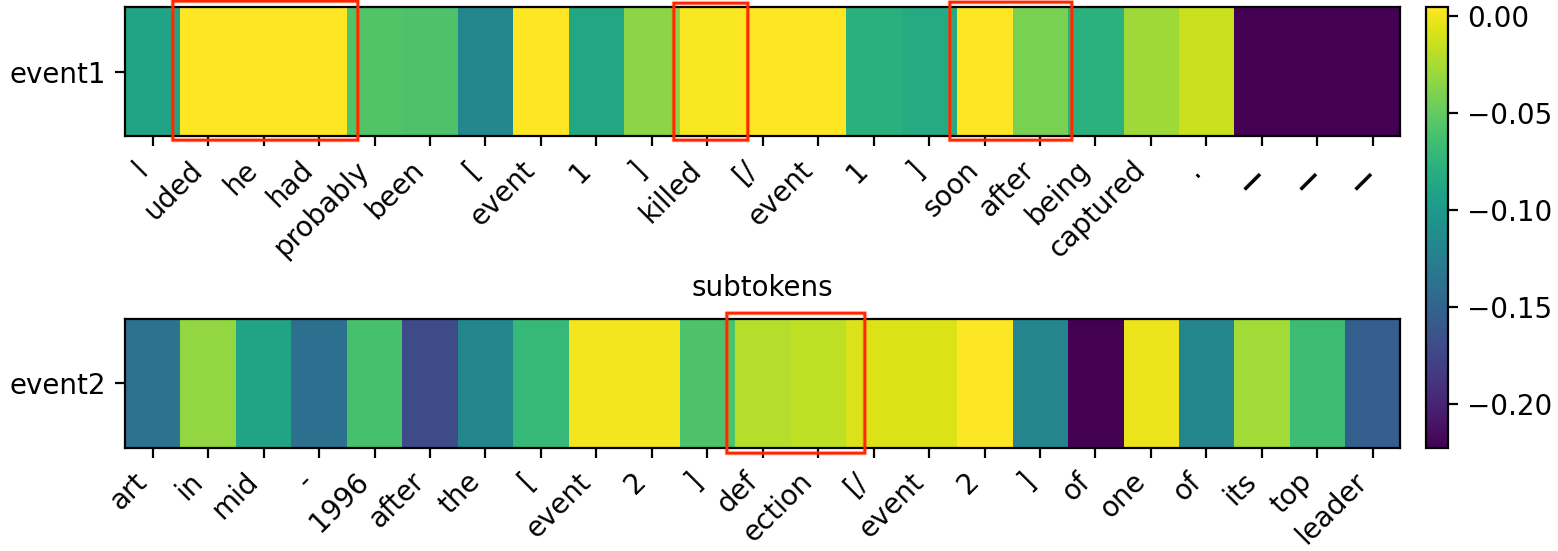}
\label{fig:adapter_original_1}
}
\hfil
\caption{Visualization of event interaction maps for evaluating the temporal relation extraction task: (a) \textbf{Full Fine-tuning}, (b) Text Concatenation with \textbf{Adapter} fine-tuning, (c) \textbf{LoRA} fine-tuning,  \textbf{Embedding Fusion} with (d) \textbf{Adapter} fine-tuning, and (e) \textbf{LoRA} fine-tuning,  \textbf{CBI} with (f) \textbf{Adapter} fine-tuning, and (j) \textbf{LoRA} fine-tuning.
\legendoutline{red} marks tokens emphasized by different methods.
\textcolor{black}{CBI offers a computationally efficient alternative to full fine-tuning while effectively capturing distributed temporal cues, outperforming other parameter-efficient methods in preserving critical signals.}
}
\label{fig:adapter_original_all}
\end{figure*}

\textcolor{black}{We perform a unified error analysis across both LoRA and adapter settings, each consisting of 500 samples, to examine whether model behavior is consistent across parameter-efficient tuning methods. Our analysis focuses on the types and structure of errors based on manual analysis.} 

\textcolor{black}{In both settings, two baselines (\textbf{Text Concatenation} and \textbf{Embedding Fusion}) show a consistent failure pattern: they often reduce containment relations to a simple order relation. In the LoRA setting, for gold \textit{IS\_INCLUDED}, Embedding Fusion and Text Concatenation misclassify \textbf{53.7\%} and \textbf{43.5\%} of cases as \textit{BEFORE}/\textit{AFTER}, respectively, while CBI reduces this to \textbf{28.7\%}. Similar patterns are observed in the adapter setting. This shows that the baselines struggle to represent nested temporal structure, while CBI better preserves it.}

\textcolor{black}{Another common error is confusion between \textit{INCLUDES} and \textit{IS\_INCLUDED}. In the LoRA setting, direction confusion accounts for \textbf{24.1\%} and \textbf{27.8\%} of containment errors for Embedding Fusion and Text Concatenation, respectively, compared to only \textbf{9.3\%} for CBI. In the adapter setting, Text Concatenation further shows a strong bias toward predicting \textit{INCLUDES} across diverse cases. These results indicate that the baselines often ignore directionality, while CBI captures it more reliably.}

\textcolor{black}{However, CBI introduces a different error pattern. For gold order relations (\textit{BEFORE}/\textit{AFTER}), CBI predicts containment in \textbf{17.2\%} of cases, compared to \textbf{13.1\%} for Embedding Fusion. Similar trends appear in the adapter setting. This suggests that stronger interaction modeling may sometimes lead to interpreting semantic relatedness as temporal containment.}

\textcolor{black}{Finally, all models struggle with the simultaneity relation. In the LoRA setting, CBI correctly predicts \textbf{20.0\%} (3/15) of SIMULTANEOUS cases, while the baselines achieve 0\%. }

\textcolor{black}{Overall, the error patterns are consistent across LoRA and adapter. The baselines tend to simplify structure (reducing containment relations to order relations or ignoring their direction).}

\subsection{Case Study}
\textcolor{black}{To understand why we need event interaction in a PEFT setting, we add one more example from the TDD-Man dataset~\citep{naik2019tddiscourse}:
Sentence 1: \textit{... luded he had probably been [event1]killed[/event1] soon after being captured.} Sentence 2: \textit{art in mid-1996 after the [event2]defection[/event2] of one of its top leader}.
The gold relation is \textit{BEFORE} and this is supported by the following cues. ``Been killed'' indicates a completed past event, suggesting temporal precedence. ``Soon after'' and ``after'' provide explicit ordering signals within the narrative. ``Defection'', anchored by  ``in mid-1996'', serves as a later temporal reference point. Together, these cues indicate that ``killed'' happened before ``defection''.}

\textcolor{black}{Model analysis shows different sensitivities to these cues. Full fine-tuning captures been ``killed'', ``after'', and ``defection''. Text Concatenation fails in both adapter and LoRA settings. Embedding Fusion captures ``defection'' with adapters, and ``been killed'', ``after'', and ``defection'' with LoRA. CBI captures ``soon after'' and ``after'' with adapters, and ``been killed'', ``soon after'', and ``defection'' with LoRA, suggesting better preservation of the distributed temporal signals.
While full fine-tuning is effective at capturing these critical cues, it comes with significantly higher computational costs and longer training times. 
In comparison, CBI provides a lightweight alternative while still being able to capture the temporal cues.}

\subsection{Wider Context of Examples}
\label{sec:context}
\textcolor{black}{The full context for the example shown in Figure~\ref{fig:adapter_original} and Figure~\ref{fig:heatmap_all} is: \textit{But while the two Slavic neighbors see themselves as natural partners, their relations since the breakup of the Soviet Union have been bedeviled by a number of disputes \_ Black Sea naval bases, border problems and Ukraine's natural gas debts. ``We have covered the entire list of questions and discussed how we will be tackling them,'' Yeltsin was quoted as saying by the ITAR-Tass news agency. ``I must say there are no unsettled problems any more. We have solved them all.'' But his sweeping statement contained no details, and gave no indication of how the [event1]disputes[/event1] could be resolved. Their solution would require a compromise between the two nations' parliaments. A major dispute concerns a broad political treaty calling for border demarcation, which the two presidents signed last May In effect, the treaty amounts to Russian recognition of Ukraine's sovereignty and borders, and the Ukrainian parliament has already ratified it. However, Russia has stalled ratification, [event2]trying[/event2] totiet it to an agreement that would permit the Russian navy to use a naval base on Ukraine's Crimean peninsula for at least 20 more years.}}

\textcolor{black}{The full context for the example shown in Figure~\ref{fig:adapter_original_all} is : \textit{January this year, British police officers who had been searching for Howes concluded he had probably been [event1]killed[/event1] soon after being captured. The Foreign Office said it had informed the family of Howes, 37 years old when he was kidnapped, that he probably died within weeks or months of his capture on March 26, 1996. Obviously, it is deeply discouraging for the family after 22 months, but there is no proof of life. But there is no evidence in either direction \_ that there is proof of life or death,'' said a Foreign Office spokesman, speaking with customary anonymity. We will continue to do everything we can to establish what has happened.'' Thai military officials who monitor Cambodian affairs said privately Thursday that Britain, through its embassies in Thailand and Cambodia, has been pushing hard to resolve the Howes case as the second anniversary of his abduction nears. Nhek Bunchhay, who had been closely involved in the search for Howes before having to flee the Cambodian capital after a coup d'etat last year, appeared confident he would find Howes' remains. He said he received information from Khmer Rouge guerrillas on where the body had been buried, and recently ordered 10 of his men from a force of 500 near Khmer Rouge headquarters in Anlong Veng to conduct the search. If and when the remains are found, he said, they would be turned over to the British Embassy, apparently meaning they would be sent across the border into Thailand and onward to Bangkok. Nhek Bunchhay said he now believed Howes had been killed within a week of his capture by a Khmer Rouge faction loyal to Pol Pot, then the guerrilla group's leader. Pol Pot is considered responsible for the radical policies that led to the deaths of as many as 1.7 million Cambodians when the communist group held power in the late 1970s. At the time Howes was captured, the Khmer Rouge were a more or less united guerrilla force with more than 10,000 men under arms. But the group began to fall apart in mid-1996 after the [event2]defection[/event2] of one of its top leaders, Ieng Sary.}}

\end{document}